%% file: main.tex
\documentclass[10pt,twocolumn,letterpaper]{article}

 \usepackage{iccv}              

\input{preamble}

\definecolor{iccvblue}{rgb}{0.21,0.49,0.74}
\usepackage[pagebackref,breaklinks,colorlinks,allcolors=iccvblue]{hyperref}

\def\paperID{6849} 
\def\confName{ICCV}
\def\confYear{2025}

\title{\textsc{mDP$^3$}: A Training-free Approach for List-wise Frame Selection in Video-LLMs}

\author{
Hui Sun\textsuperscript{1,2}, Shiyin Lu\textsuperscript{3}, Huanyu Wang\textsuperscript{1,2}, Qing-Guo Chen\textsuperscript{3}, \\
Zhao Xu\textsuperscript{3}, Weihua Luo\textsuperscript{3}, Kaifu Zhang\textsuperscript{3}, Ming Li\textsuperscript{1,2,}\thanks{Ming Li is the corresponding author.} \\
\textsuperscript{1}National Key Laboratory for Novel Software Technology, Nanjing University, China \\
\textsuperscript{2}School of Artificial Intelligence, Nanjing University, China \\
\textsuperscript{3}Alibaba Group, Hangzhou, China \\
{\tt\small 
	\{sunh, wanghy, lim\}@lamda.nju.edu.cn
} \\
{\tt\small 
	\{running.lsy, qingguo.cqg, changgong.xz, weihua.luowh, kaifu.zkf\}@alibaba-inc.com}
}

\begin{document}
\maketitle

\input{sec/0_abstract}
\input{sec/1_intro}

\input{sec/2_related_work}

\input{sec/3_method}

\input{sec/4_experiment}

\input{sec/5_conclusion}
\section*{Acknowledgements}
This paper is supported by NSFC~(62076121)
and Major Program~(JD) of Hubei Province~(2023BAA024).
The authors would like to thank Hao-Yuan He for his helpful feedback on drafts of the paper.

{
    \small
    \bibliographystyle{ieeenat_fullname}
    \bibliography{main}
}
\input{sec/X_suppl}

\end{document}

%% file: preamble.tex
\usepackage[table]{xcolor}
\usepackage{multirow}
\usepackage{multicol}
\usepackage{makecell}
\usepackage{arydshln}
\usepackage{amssymb}
\usepackage{bbding}
\usepackage{pifont}
\usepackage{caption}
\usepackage{titletoc}
\usepackage{amsthm}

\usepackage[linesnumbered,ruled,vlined]{algorithm2e}

\newcommand{\mname}{\mbox{\textsc{mDP$^3$}\xspace}}
\newcommand{\unk}{~~~-~~~}

\newcommand{\N}{\mathbb{N}}
\newcommand{\mL}{\mathbf{L}}
\newcommand{\I}{\mathbf{I}}
\newcommand{\mr}{\mathbf{r}}

\newcommand{\T}{\mathcal{T}}
\newcommand{\R}{\mathcal{R}}

\newcommand{\nocheckmark}{\ding{55}}
\newcommand{\semicheckmark}{{\ding{51}\kern-0.65em\ding{55}}}

\definecolor{green}{HTML}{34A853}
\definecolor{red}{HTML}{EA4335}
\newcommand{\gainfoot}[1]{\scriptsize\color{green}\textbf{#1$\uparrow$}}
\newcommand{\lossfoot}[1]{\scriptsize\color{red}\textbf{#1$\downarrow$}}
\newcommand{\emptyfoot}[1]{\scriptsize\textbf{\phantom{#1$\uparrow$}}}

\newtheorem{definition}{Definition}
\newtheorem{theorem}{Theorem}
\newtheorem{lemma}{Lemma}

\usepackage[normalem]{ulem}


%% file: sec/0_abstract.tex
\begin{abstract}
Video large language models~(Video-LLMs) have made significant progress in understanding videos.
However, processing multiple frames leads to lengthy visual token sequences,
presenting challenges such as the limited context length cannot accommodate the entire video,
and the inclusion of irrelevant frames hinders visual perception.
Hence, effective frame selection is crucial.
This paper emphasizes that frame selection should follow three key principles: query relevance, list-wise diversity, and sequentiality.
Existing methods, such as uniform frame sampling and query-frame matching, do not capture all of these principles.
Thus, we propose Markov decision determinantal point process with dynamic programming~(\mname{}) for frame selection,
a training-free and model-agnostic method that can be seamlessly integrated into existing Video-LLMs at test time.
Our method first estimates frame similarities conditioned on the query using a conditional Gaussian kernel within the reproducing kernel Hilbert space~(RKHS).
We then apply the determinantal point process~(DPP) to the similarity matrix to capture both query relevance and list-wise diversity.
To incorporate sequentiality, we segment the video and apply DPP within each segment, conditioned on the preceding segment selection, modeled as a Markov decision process~(MDP) for allocating selection sizes across segments. 
Theoretically, \mname{} provides a \((1 - 1/e)\)-approximate solution to the NP-hard list-wise frame selection problem with pseudo-polynomial time complexity, demonstrating its efficiency.
Empirically, \mname{} significantly outperforms existing methods, verifying its effectiveness and robustness.
\end{abstract}

%% file: sec/1_intro.tex
\vspace{-0.3cm}
\section{Introduction}
\label{sec:intro}

Large language models~(LLMs) have exhibited remarkable performance across various natural language processing tasks~\cite{openai2023chatgpt,yang2024qwen2,dubey2024llama,anthropic2024introducing}.
Concurrently, transformers for visual tasks~(\eg, ViT~\cite{dosovitskiy2021an}) and cross-modal visual-language models~(VLMs), 
such as CLIP~\cite{radford2021learning} and BLIP~\cite{li2022blip}, have advanced rapidly.
Riding the waving of LLMs, multimodal large language models~(MLLMs) are transforming the landscape of computer vision globally~\cite{wang2024videoagent,Qwen2VL,lu2024ovis,wang2024videotree,TanKoala2024}.

MLLMs typically use the visual encoder from VLMs and a new-trained projector to extract a long sequence of visual tokens from images.
The combined sequence of visual and text tokens is then fed into LLMs for visual understanding~\cite{liu2023llava,cheng2024videollama2,li2023mvbench}.
However, a video contains hundreds or even thousands of frames. 
Directly concatenating sequences from multiple frames and feeding them into Video-LLMs poses several challenges: 
1)~LLMs' limited context length can not accommodate excessively long visual sequences;
2)~Redundant and duplicate frames can hinder visual perception, as irrelevant frames may cause ``lost-in-the-middle'' issues~\cite{liu2024st,yu2024frame}.
3)~Edge-deployed LLMs face resource constraints when processing overly long input sequence, and
4)~Proprietary models charge based on token usage, making excessive frame input costly.

\begin{figure}[t]
	\centering
  	\includegraphics[width=0.8\linewidth]{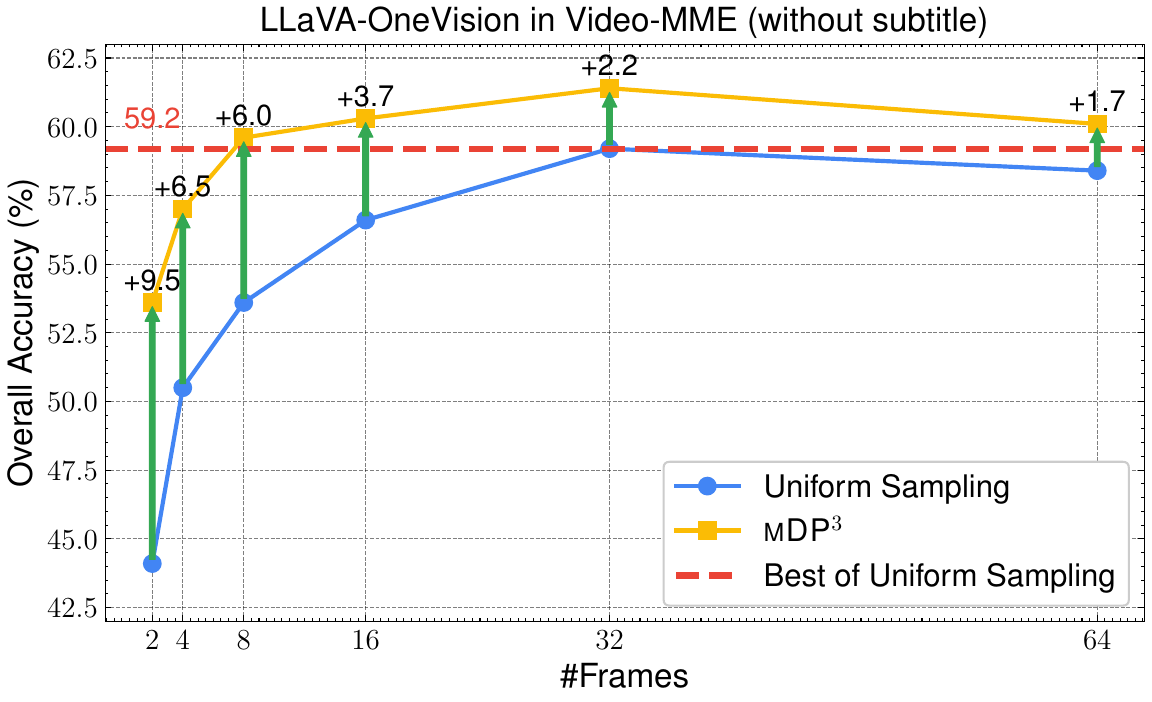}
	\caption{
	    Overall accuracy (\%) on Video-MME~(without subtitles) for LLaVA-OneVision, 
    	comparing uniform sampling and \mname{}\ selected from 128 frame candidates.
	}
	\label{fig:gain_on_videomme}
\end{figure}

As shown in the empirical example in \cref{fig:gain_on_videomme}, 
the blue line illustrates that the marginal benefits of increasing the number of frames diminish, 
with performance even declining when using 64 frames instead of 32 frames. 
This indicates a law of diminishing marginal utility as the number of frames increases, 
and additional frames may even have negative effects. 
Hence, selecting a limited number of key frames is crucial for effective and efficient video understanding.

In previous research, uniform frame sampling is a common strategy in Video-LLMs~\cite{miao2023towards,lin2023video,zhang2024video,wang2024internvideo2}. 
However, it is arbitrary, especially in video question-answering~(VidQA) tasks, 
as it disregards the current query, potentially missing important key frames and including irrelevant ones. 
Furthermore, \citet{liang2024keyvideollm,wang2024weakly,yu2024self} have explored using the query to retrieve the most relevant frames.
However, these point-wise frame selection methods overlook the list-wise relationships among selected frames, including: 
1)~\emph{Diversity}: Videos often contain many similar frames; redundantly selecting them increases the computational burden during inference with little information gain and may hinder the usage of text information.
2)~\emph{Sequentiality}: The inherent sequential nature of frames is essential. 
For instance, when asking, ``What did the chef take \underline{after} adding the spices?'' in a cooking video, frames depicting the steps following the addition of the spices are essential.

Building on the insights from the above analysis, we emphasize that effective frame selection in Video-LLMs should adhere to three key principles:
\textbf{query relevance, list-wise diversity, and sequentiality.} 
Prior research has addressed some of them but has not fully encompassed all of them, particularly the latter two~\cite{jin2024chat,wang2025ViLA,tang2025,hu2025m}.
Modeling list-wise relationships in subset frame selection presents two main challenges:
1)~the lack of available supervised list-wise datasets; 
2)~even if a dataset is constructed to learn a list-wise scoring function, finding the optimal subset with the highest score remains an NP-hard problem that cannot be solved in polynomial time complexity~\cite{bian2020efficient,chen2018fast}.

In this paper, we propose a training-free, model-agnostic list-wise frame selection method called 
\underline{M}arkov~\underline{D}ecision \underline{D}eterminantal~\underline{P}oint \underline{P}rocess~with \underline{D}ynamic~\underline{P}rogramming, abbreviated as \mname{}.
We reuse pretrained VLMs to compute list-wise score for the subset using Determinantal Point Process~(DPP), 
thereby eliminating the need for collecting a supervised dataset or training a scoring model from scratch.
DPP was originally introduced to characterize the Pauli exclusion principle, where no two fermions can occupy the same quantum state.
This repulsion, interpreted as list-wise diversity, is effectively modeled by DPP~\cite{macchi1975coincidence},
which can obtain a $(1-1/e)$-approximate solution with polynomial complexity~\cite{chen2018fast}.
In \mname{}, we extend standard DPP to consider query relevance, list-wise diversity, and sequentiality for frame selection in VidQA.

Specifically, we define a conditional multiple Gaussian kernel~(CMGK) to estimate high-dimensional frame similarities conditioned on the query within the reproducing kernel Hilbert space~(RKHS), rather than using a query-agnostic low-dimensional measure such as cosine similarity. 
We then apply DPP to the conditional similarity matrix generated by CMGK to capture the relevance of frames to the current query and ensures list-wise diversity among them. 
To incorporate sequentiality, we partition the video into continuous segments and apply DPP within each segment, conditioned on the selection from the immediately preceding one.
This process is modeled as a Markov Decision Process~(MDP) to optimize the allocation of the total selection size $k$ across segments, resulting in a solution with pseudo-polynomial time complexity.

To evaluate \mname{}, we integrate it as a plug-and-play process at the test time of state-of-the-art~(SOTA) Video-LLMs, including VILA-V1.5~\cite{lin2024vila}, MiniCPM-V2.6~\cite{yao2024minicpm}, LLaVA-OneVision~\cite{li2024llava}, and Ovis2~\cite{ovis2}.
We conduct extensive experiments on three widely used long-video benchmarks: Video-MME~\cite{fu2024video}, MLVU~\cite{MLVU}, and LongVideoBench~\cite{wu2024longvideobench}.
Experimental results show significant performance improvements over baseline methods.

Furthermore, we summarize our main contributions:
\begin{enumerate}
    \item We demonstrate that effective frame selection should account for query relevance, list-wise diversity, and sequentiality. The proposed \mname{} is the first to address all three aspects comprehensively.
    \item \mname{} is effective, offering a \((1-1/e)\)-approximate solution to the NP-hard list-wise subset frame selection problem, achieving significant performance improvements over existing baselines.
    \item \mname{} is efficient, extending DPP methods by optimally allocating the selection capacity \(k\) across the video for VidQA, maintaining pseudo-polynomial time complexity, even though this allocation problem is NP-hard.
    \item \mname{} is training-free and model-agnostic, enabling seamless plug-and-play integration with various existing Video-LLMs at test time.
\end{enumerate}

%% file: sec/2_related_work.tex
\section{Realted Work}\label{sec:related_work}

MLLMs have achieved impressive performance in visual understanding tasks by utilizing VLMs to process an image into a sequence of visual tokens, with lengths ranging from 64 to 729~\cite{radford2021learning,zhai2023sigmoid}, which are then input into LLMs~\cite{liu2024llavanext,yang2024qwen2,li2024llava,yao2024minicpm,chen2024far}.
However, in video understanding, a video contains hundreds or even thousands of frames, resulting in a long sequence of visual tokens.
As a result, feeding an entire video into Video-LLMs becomes impractical. 
Some research has explored extending the context length of LLMs~\cite{wan2023efficient,zhang2024longva,xiong2024effective,xue2024longvila,song2024moviechat},
but these methods still struggle with computational burden and the ``lost-in-the-middle'' issue, where important context may be overlooked in long sequences~\cite{liu2024st,yu2024frame}.

Most existing Video-LLMs rely on uniform frame sampling~\cite{chen2023internvl,li2023videochat,liu2024oryx,shu2024video}, 
which arbitrarily disregards the current query, often missing relevant frames and including irrelevant ones.
Traditional key-frame methods that detect significant transitions between adjacent frames also face this limitation~\cite{sheena2015key,nasreen2013key,wolf1996key}.
Some studies have attempted to use the query to retrieve the top-$k$ most relevant frames with VLMs but fail to consider frame diversity~\cite{wang2024weakly,yu2024self},
often resulting in redundant frame selection, and increased computational burden without meaningful information gain.
Some cluster-based methods attempt to regularize pair-wise diversity among frames but still fail to capture combinatorial relationships in a list-wise manner~\cite{keplerlab_katna,jin2024chat}. 
Recently, \citet{yu2024frame} proposed a frame selection method tailored for VILA~\cite{lin2024vila}, 
which requires collecting a supervised dataset and end-to-end training of a selector, limiting its generalizability to other Video-LLMs.
Furthermore, these method does not explicitly address list-wise diversity or sequentiality, both of which are essential for effective frame selection.

The DPP was introduced by~\citet{macchi1975coincidence} to model the distribution of fermions in thermal equilibrium, where no two fermions can occupy the same quantum state, resulting in an ``anti-bunching'' effect that can be interpreted as diversity and effectively modeled using DPP. 
DPP has been widely applied to characterized list-wise diversity in subset selection problems~\cite{kulesza2012determinantal,moss2023inducing,celis2018fair,li2024evaluation,chen2018fast}. 
However, standard DPP does not account for sequentiality. 
\citet{gong2014diverse} addressed this by segmenting the video and modeling local diversity within each segment conditioned on the selections from the preceding segment.
Nevertheless, allocating the total selection size \(k\) across segments remains challenging.
\citet{zheng2021k} treat it as an integer programming problem, relaxing problem constraints to find an approximate solution.
In contrast, we model DPP for sequentiality as an MDP, enabling optimal allocation of \(k\) through a dynamic programming algorithm with pseudo-polynomial time complexity. 
This approach avoids constraint relaxation and can find the optimal allocation efficiently in practice.

%% file: sec/3_method.tex
\section{Method}\label{sec:method}
In this paper, we focus on VidQA, which can be generalized to other video understanding tasks such as video summarization and grounding~\cite{xiao2024can,liu2025timecraft}.
This task is formulated as \((V, q) \mapsto \text{answer}\), where \(V=\{f_i\}_{i=1}^{n}\) is a video with \(n\) frames, \(f_i\) is the \(i\)-th frame, and \(q\) is the query.
The index set of frames is denoted by \(\N_n := \{1, 2, \dots, n\}\), 
and the list-wise $k$ frame selection can be expressed as: 
\begin{equation}\label{eq:score}
    S^* = \mathop{\mathrm{argmax}}_{S \subseteq \N_n , \left|S\right|=k}\ \mathrm{Score}(S)\,.
\end{equation}
This problem presents two challenges:
1) How can we design an effective \(\mathrm{Score}\) function that simultaneously addresses query relevance, list-wise diversity, and sequentiality?
2) How can we tackle the NP-hard challenges of the list-wise subset frame selection problem to find an effective solution with practical time complexity?

In \cref{subsec:method_sim}, we introduce the conditional multiple Gaussian kernel~(CMGK) within the reproducing kernel Hilbert space~(RKHS) to estimate high-dimensional similarities between frames conditioned on the current query.
In \cref{subsec:method_dpp}, we employ the determinantal point process~(DPP) based on the conditional similarity matrix of the frames, capturing both query relevance and list-wise frame diversity. 
Moreover, in \cref{subsec:method_kdp}, we partition the video into continuous segments and model frame selection as a Markov decision process~(MDP) to ensure list-wise diversity within local segments while maintaining sequentiality across segments.

\subsection{High-dimensional Conditional Similarity}\label{subsec:method_sim}
With the success of VLMs~\cite{radford2021learning,zhai2023sigmoid,alayrac2022flamingo},
frames and queries can be mapped into a unified latent embedding space.
For conciseness, we reuse \(f_i\) and \(q\) to denote the embeddings of the \(i\)-th frame and the query, respectively, to avoid defining redundant symbols. 
This is expressed as:
\begin{equation}\label{eq:emb} 
	f_i \leftarrow \mathrm{VLM_{vision}}(f_i) \in \mathbb{R}^d, \ q \leftarrow \mathrm{VLM_{text}}(q) \in \mathbb{R}^d \,,
\end{equation}
where \(\mathrm{VLM_{vision}}()\) and \(\mathrm{VLM_{text}}()\) denote the vision and text encoders of the VLM, respectively.
Notably, since the vision encoder is commonly used in MLLMs, the additional parameters introduced by the text encoder are negligible.

Instead of directly estimating similarity using cosine similarity or the inner product of embeddings, we map the embeddings into an RKHS via the reproducing kernel feature map \(\phi \in \mathcal{H}_k\), 
where, according to the kernel trick~\cite{scholkopf2005support}, the high-dimensional similarity in RKHS is expressed by \(\left<\phi(x), \phi(y)\right> = k(x, y)\).
We employ a characteristic multi-kernel defined as a convex combination of multiple positive semi-definite~(PSD) kernels \(\{k_u\}\)~(\eg Gaussian kernel): 
\begin{equation}\label{eq:k_family}
	\mathcal{K} \triangleq \left\{ 
		k = \sum_{u=1}^{U} \beta_u k_u : \sum_{u=1}^{U} \beta_u = 1, \beta_u \geq 0, \forall u 
	\right\}\,.
\end{equation}
The multi-kernel utilizes various kernels to enhance the estimation of high-dimensional similarity, thereby providing a principled method for optimal kernel selection~\cite{long2015learning,sun2023enhancing}.

In VidQA tasks, both frame-query relevance and list-wise diversity among the selected frames are important. 
Therefore, estimating frame similarity should consider the query relevance.
To address this, we propose the CMGK as:
\begin{equation}\label{eq:cond_k}
    \tilde{k}(f_i, f_j | q) = g(f_i, q)k(f_i, f_j)g(f_j, q),\quad g, k \in \mathcal{K}\,.
\end{equation}
This conditional kernel refines frame similarity by considering the query, reducing the importance of irrelevant frames.

Consequently, we construct a high-dimensional conditional similarity matrix that captures the similarities among frames while considering the query relevance, expressed as:
\begin{equation}\label{eq:tilde_L}
    \tilde{\mL} \in \mathbb{R}^{n \times n}, \quad \tilde{L}_{ij} = \tilde{k}(f_i, f_j | q)\,.
\end{equation}
This can also be expressed as:
\begin{equation}\label{eq:cmgk}
    \tilde{\mL} = \mathrm{diag}(\mr)\cdot \mL \cdot \mathrm{diag}(\mr)\,,
\end{equation}
with frame similarities \(\mL \in \mathbb{R}^{n \times n}, \mL_{ij} = k(f_i, f_j), k \in \mathcal{K}\), and frame-query relevances \(\mr \in \mathbb{R}^n, \mr_i = g(f_i, q), g \in \mathcal{K}\).

\subsection{Determinantal Point Process~(DPP)}\label{subsec:method_dpp}

DPP was initially introduced to model fermion repulsion in quantum physics and has proven effective in modeling list-wise diversity in machine learning~\cite{ye2023compositional,chen2018fast}.
We apply DPP to the high-dimensional conditional similarity matrix \(\tilde{\mL}\) to capture both query relevance and list-wise diversity in frame selection.
In fact, it represents the geometric volume of the subset in the RKHS; a larger volume indicates a more well-spread subset, leading to better list-wise diversity.

Formally, DPP \(\mathcal{P}\) serves as a probability measure over the \(2^n\) subsets of frames sampled without replacement from a video \(V = \{f_i\}_{i=1}^{n}\).
The probability of selecting a subset of frame indices \(S \subseteq \N_n\) is defined as:
\begin{equation}\label{eq:dpp}
	\mathcal{P}(S)=\frac{\det(\tilde{\mL}_S)}{\det({\tilde{\mL} + \I})} \propto \det(\tilde{\mL}_S) \,,
\end{equation}
where \(\tilde{\mL}_S \equiv [\tilde{\mL}_{ij}]_{i,j \in S} \in \mathbb{R}^{|S|\times|S|}\) is a PSD matrix denoting the submatrix of \(\tilde{\mL}\) corresponding to the selected frame indices in \(S\).
Based on \cref{eq:cmgk} and \cref{eq:dpp}, we can derive the unnormalized log-probability of the subset \(S\) as:
\begin{equation}
	\log\left(\det(\tilde{\mL}_S)\right) = \sum_{i\in S}\log\left(\mr_i^2\right) + \log\left(\det(\mL_S)\right)\,,
\end{equation}
which clearly shows how to capture query relevance~(\ie \(\mr\)) and list-wise frame diversity~(\ie \(\det(\tilde{\mL}_S)\)) in frame selection.
Without loss of generality, we introduce a weighting factor to trade-off query relevance and list-wise diversity:
\begin{equation}\label{eq:log_dpp}
	\log \left( \det(\tilde{\mL}_S) \right) = \frac{1}{\lambda} \sum_{i \in S} \log\left(\mr_i^2\right) + \log\left(\det(\mL_S)\right)\,.
\end{equation}
This is a more general form of the standard DPP, equivalent to scaling the bandwidth \(h\) in the Gaussian kernel \mbox{\(k(x,y) = \exp\left(-\frac{\| (x-y)/h \|^2}{2\sigma^2}\right)\)} from \(h\) to \(\sqrt{\lambda} \cdot h\).

Frame selection follows the law of diminishing marginal utility, 
as the marginal gain from selecting additional frames is non-increasing~(as shown in \cref{fig:gain_on_videomme}).
This phenomenon is formally known in statistics as \emph{submodularity}, 
which is a well-established property of DPP, as shown in previous research~\cite{fujishige2005submodular,krause2014submodular}.
Selecting the optimal subset to maximize the submodular score function~\cref{eq:dpp,eq:log_dpp}, known as the submodular maximization problem in subset selection, is generally NP-hard.
Fortunately, there is a popular greedy algorithm can solve this problem,
which offers polynomial complexity and guarantees a \((1 - 1/e)\)-approximation, \ie, the solution is at least \((1 - 1/e)\) of the optimal score~\cite{chen2018fast}.

Specifically, we employ maximum a posteriori~(MAP) inference~\cite{chen2018fast,ye2023compositional} to solve DPP~(corresponding to \cref{eq:dpp,eq:log_dpp}).
We initialize the selected subset as \(S \leftarrow \emptyset\).
In each iteration, we select a new frame that maximizes the marginal gain in log-probability, given by:
\begin{equation}\label{eq:greedy}
	j = \mathop{\mathrm{argmax}}\limits_{i\in \mathbb{N}_n \backslash S} \log\left(\det(\tilde{L}_{S \cup \{i\}})\right) - \log\left(\det(\tilde{L}_S)\right)\,.
\end{equation}
The selected frame index \(j\) is then added to \(S\) until the subset size reaches the capacity limit \(k\).
By using Cholesky decomposition, the time complexity is reduced from \(\mathcal{O}(nk^3)\) to \(\mathcal{O}(nk)\) per iteration through incremental updates to the Cholesky factor~\cite{chen2018fast}.
Therefore, the overall time complexity of the MAP inference is \(\mathcal{O}(nk^2)\), compared to the vanilla KNN, the additional latency is negligible while \(k \ll n\)~\cite{ye2023compositional}.

\subsection{Markov Decision DPP}\label{subsec:method_kdp}

As described above, applying DPP to the conditional similarity matrix effectively captures both query relevance and list-wise diversity.
However, standard DPP treats video frames as independently permutable, disregarding the sequential structure inherent in video.
To address this, existing works on video summarization~\cite{gong2014diverse,zheng2021k} segment a video into multiple consecutive short segments and apply DPP within each segment, conditioned on the selections from the preceding segment.
This approach effectively models sequentiality by \emph{enhancing diversity among neighboring frames while reducing it for temporally distant ones.}
In VidQA, frames relevant to the query are often concentrated in specific segments rather than being uniformly distributed across the video. 
This presents a new challenge in allocating the total selection size \(k\) across multiple segments.
Focusing on video summarization, \citet{gong2014diverse} allocate \(k\) uniformly across segments,
while \citet{zheng2021k} formulate it as an integer programming problem and relax the constraints to find an approximate solution.
Instead, we focus on general VidQA and model this sequential DPP across the segment sequence as an MDP. 
We propose a dynamic programming algorithm that can optimally allocates the total selection size \(k\) across segments with pseudo-polynomial complexity, without relaxing problem constraints.

First, we segment a video into multiple segments of equal length \(m\), so that a video with \(n\) frames is divided into \(T = \lceil n/m \rceil\) segments. 
The corresponding frame indices of the video, \(\N_n\), are divided as \(\N_n = \bigcup_{t=1}^{T} N_t\), 
where \(N_t\) denotes the frame indices of the \(t\)-th segment, with $|N_t|=m$.
Moreover, let the selected subset from the \(t\)-th segment be denoted as \(S_t \subseteq N_t\).
Therefore, when applying DPP to the \(t\)-th segment, conditioned on the selection in the \((t-1)\)-th segment \(S_{t-1}\), 
the DPP conditional distribution~\cite{gong2014diverse} for the current segment \(S_t\) is expressed as:
\begin{equation}\label{eq:cond_dpp}
    \mathcal{P}(S_t | S_{t-1}) = \frac{\det(\tilde{\mL}_{S_{t-1} \cup S_t})}{\det(\tilde{\mL}_t + \I_t)} \,,
\end{equation}
where \(\tilde{\mL}_t\) is the submatrix of \(\tilde{\mL}\) for the indices \(S_{t-1} \cup N_t\), and \(\I_t\) is a diagonal matrix with zeros for \(S_{t-1}\) and ones for \(N_t\).
Thus, the joint distribution of subsets can be written as:
\begin{equation}\label{eq:seq_score}
    \mathcal{P}(S_1, S_2, \dots, S_T) 
        = \mathcal{P}(S_1) \prod_{t=2}^{T} \mathcal{P}(S_t \mid S_{t-1}) \,.
\end{equation}
If the selection size for the \(t\)-th segment is predetermined as \(|S_t| = k_t\),
the MAP inference in \cref{eq:seq_score} can be performed sequentially, from the first to the last~\cite{gong2014diverse}, similar to Bayesian belief updates (such as Kalman filtering):
\begin{equation}\label{eq:map_inf}
    S_t^* = \mathop{\mathrm{argmax}}_{S_t \subseteq N_t, |S_t|=k_t} \mathcal{P}(S_t \mid S_{t-1}^*) \,.
\end{equation}
This implies that when determining one of \(k_t\) and \(S_t^*\), the other is uniquely determined, \ie, \(k_t \leftrightarrows S_t^*\).

Determining the optimal allocation of selection size \(k_t\) for each segment is challenging. 
We observe that the MAP inference~(\cref{eq:map_inf}) of the DPP conditional distribution~(\cref{eq:cond_dpp}) uniquely depends on a triplet state, consisting of:
1)~the candidate indices in the current segment \(N_t\),
2)~the selections from the preceding segment \(S_{t-1}^*\), determined by $k_{t-1}$, and
3)~the remaining capacity for further selections \(k - C_{t-1}\), where \(C_{t-1}\) denotes the total selection size before the \((t-1)\)-th~(inclusive) segment.
Consequently, this MAP inference can be formulated as an MDP, where the state is defined as the triplet \((N_t, k_{t-1}, C_{t-1})\).
Moreover, the segment index \(t-1\) can uniquely determine \(N_t\).
Therefore, the state can be represented as \((t \in [1, T], k_t \in [0, k], C_t \in [0, k])\).
The state transition from \(t-1\) to \(t\) is deterministic when the action of selecting \(k_t\) frames from the \(t\)-th segment is given, which can be expressed as:
\begin{equation}
	(t-1, k_{t-1}, C_{t-1})
	\xrightarrow{\hspace{-0.1cm}
		\scriptsize\begin{array}{c}
	        \text{selecting } S_t \\
	        \text{with } |S_t|=k_t
    	\end{array}
	\hspace{-0.1cm}}
	(t, k_t,\overbrace{C_{t-1}+ k_t}^{C_t})\,.
\end{equation}
The corresponding reward in MDP is defined as
\begin{equation}\label{eq:reward}
    \R_{S_t}(t-1, k_{t-1}, C_{t-1}) = \log \left(\mathcal{P}(S_t \mid S_{t-1}^*)\right) \,,
\end{equation}
where \(S_{t-1}^*\) is uniquely determined by \(k_{t-1}\) in the state $(t-1, k_{t-1}, C_{t-1})$.
This reward can be efficiently computed using a greedy DPP algorithm applied to the conditional DPP (\ie, \cref{eq:cond_dpp}).
Therefore, the list-wise \(\mathrm{Score}\) function in \cref{eq:score}, based on \cref{eq:seq_score}, is defined as:
\begin{equation}\label{eq:opt_score}
	\mathrm{Score}(S_1,S_2,\dots,S_T) =
		\sum_{t=0}^{T-1} \R_{S_{t+1}}(t, k_t, C_t)\,.
\end{equation}
Using this definition and the \((1 - 1/e)\)-approximate greedy DPP algorithm, we can apply a dynamic programming approach to identify a solution \(\hat{S}\) such that \mbox{\(\mathrm{Score}(\hat{S}) \geq (1 - 1/e) \cdot \mathrm{Score}(S^*)\)}, with a closely \(\mathcal{O}(nk^4)\) time complexity~(proof provided in the Appendix).

We define a value function \(Q(t, k_t, C_t)\) as the maximum accumulated reward from the initial state \((0, 0, 0)\) to the current state \((t, k_t, C_t)\). 
Consequently, 
\begin{equation}
    Q^*(t, C_t) = \max\limits_{0 \leq k_t \leq k - C_t} Q(t, k_t, C_t)
\end{equation}
represents the maximum accumulated reward for selecting \(C_t\) frames up to and including the \(t\)-th segment. 
The solution to \cref{eq:opt_score} is given by \(\hat{S} = \mathop{\mathrm{argmax}}_{S} Q^*(T, k)\).

\input{sec/X_algorithm.tex}

\input{tables/main_results}

Dynamic programming can be used to update the \(Q\) score and identify the solution \(\hat{S}\) with an optimal allocation of \(k\), with a time complexity closely \(\mathcal{O}(nk^4)\).
In practice, to improve efficiency further, we adopt lazy dynamic programming strategies.
First, we directly maintain the \(Q^*\) score instead of updating \(Q\).
This strategy assumes that selecting a frame from the \(t\)-th segment depends only on the preceding state with optimal $k_{t-1}^*$, \ie, \((t-1, k_{t-1}^*, C_{t-1})\), corresponding to \(Q^*(t-1, C_{t-1})\), similar to the online MAP inference described in \cref{eq:map_inf}~\cite{gong2014diverse}.
Second, only the most recently selected frame is used as the condition in~\cref{eq:reward}, approximating an adaptive condition size and a fixed candidate segment size \(m\) for the \(t\)-th segment.
Consequently, with the lazy strategies, the state is simplified to \((t, C_t)\).
The policy function is denoted by \(S_t^*=\pi(t, C_t)\), and \(\T_{t, C_t}\) is an initially empty list \([\,]\) that tracks the selection trace \([S_1, S_2, \dots, S_t]\).
The corresponding reward simplifies to:
\begin{equation}
	\R_{S_t}^*(t-1,C_t-k_t)=\log (\mathcal{P}(S_t \mid \T_{t-1, C_{t-1}}[-1])), 	
\end{equation}
where $C_t-k_t = C_{t-1}$.
For brevity, we denote $Q^*_{t, C_t}=Q^*(t, C_t)$.
The dynamic programming update is:
\begin{align}\label{eq:dp_update}
	&Q^*_{t, C_t} = \max_{S_t \subseteq N_t} Q^*_{t-1, C_t-k_t} + \R_{S_t}^*(t-1,C_t-k_t)\,, \notag \\
	&\pi(t, C_t)  = \mathop{\mathrm{argmax}}_{S_t \subseteq N_t} 
					Q^*_{t-1, C_t-k_t} + \R_{S_t}^*(t-1,C_t-k_t)\,, \notag \\
	&\T_{t, C_t} = \T_{t, C_t - |S_t^*|} + [S_t^*],\text{where } S_t^* = \pi(t, C_t)\,.
\end{align}
The final solution is \(\hat{S} = \mathop{\mathrm{argmax}}_{S} Q^*(T, k) = \T(T, k) \).

The pseudocode for \mname{} is provided in \cref{alg}.
The end index \((t, C_t + k_t)\) is iterated, while the start index \((t-1, C_t)\) remains fixed.
The update functions are equivalent with \cref{eq:dp_update}.
This implementation is more efficient as the greedy DPP algorithm (line 9) can be computed incrementally in $\mathcal{O}(mk)$.
Thus, the worst-case time complexity of \cref{alg} is closely \(\mathcal{O}(nk^3)\) in practice~(proof in Appendix).
Compared to the standard DPP complexity of \(\mathcal{O}(nk^2)\), the additional factor \(k \ll n\) is negligible. 
Additionally, the iteration in line 6 is independent, enabling parallel updates.
This results in a time complexity of \(\mathcal{O}(nk^2)\), making its efficiency comparable to the standard DPP.

%% file: sec/X_algorithm.tex
\begin{algorithm}[!t]
    \caption{\label{alg} \mname}
    \KwIn{Video $V = \{f_i\}_{i=1}^{n}$, query $q$, select size $k$.}
    \KwOut{A set of frame indices $S$ with $|S| = k$.}
    
    Initialize $Q_{0,0}^* = 0$ and $\T_{0,0} = [\,]$\;
    Extract embeddings $f_i$ and $q$ using \cref{eq:emb}\;
    
    \For{$t \gets 1$ to $T = \lceil n/m \rceil$}{
        Compute $\mr_t$, $\mL_t$, and $\tilde{\mL}_t$ using \cref{eq:k_family,eq:cond_k,eq:tilde_L,eq:cmgk}\;
        Compute offset $\mathrm{o} = -\log(\det(\tilde{\mL}_t + \I_t))$\;
        
        \For{$C_{t-1} \gets 0$ to $k$}{
            \For{$k_{t-1} \gets 0$ to $\min(m, k, k - C_{t-1})$}{
            	Get $C_t = C_{t-1}+k_{t-1}$\;
	            Get $\mathrm{rwd}^* = \R_{S_t}^*(t-1, C_{t-1})$ and $S_t = \pi(t, C_t)$ using DPP as \cref{eq:greedy,eq:cond_dpp,eq:seq_score,eq:reward}\;
                Compute $cur\_q = Q_{t-1, C_{t-1}} + \mathrm{rwd}^*$\;

                \If{$cur\_q > Q_{t, C_t}$}{
                    Update $Q_{t, C_t} \gets cur\_q$\;
                    Update $\T_{t, C_t} \gets \T_{t-1, C_{t-1}} \cup S_t$\;
                }
            }
        }
    }
    \Return $\T_{T, k}$\;
\end{algorithm}

%% file: tables/main_results.tex
\begin{table*}[!t]
\centering
\begin{tabular}{lcccccccc}
    \toprule
		\multirow{2}{*}{\textbf{Model}} & 
		\multirow{2}{*}{\makecell{\textbf{LLM} \\ \textbf{Size}}} & 
		\multirow{2}{*}{\textbf{\#Frames}} & 
		\multicolumn{4}{c}{\textbf{Video-MME$_{(wo/w\ subs)}$}} & 
		\multirow{2}{*}{\textbf{MLVU}} &
		\multirow{2}{*}{\textbf{LVB$_{val}$}} \\

		\cmidrule(lr){4-7} & & & Overall & Short & Medium & Long & & \\ 
		
		\multicolumn{2}{c}{\textit{Avg. Video Duration}} && \textit{17min} & \textit{1.3min} & \textit{9min} & \textit{41min} & \textit{12min} & \textit{12min} \\
  	\toprule
	Video-LLaVA     & 7B & 8 & 39.9 / 41.6 & 45.3 / 46.1 & 38.0 / 40.7 & 36.2 / 38.1 & 47.3 &   - \\
	Qwen-VL-Chat    & 7B & 8 & 41.1 / 41.9 & 46.9 / 47.3 & 38.7 / 40.4 & 37.8 / 37.9 &   -  &   -  \\
	VideoChat2      & 7B & 8 & 39.5 / 43.8 & 48.3 / 52.8 & 37.0 / 39.4 & 33.2 / 39.2 & 44.5 &   -  \\
	Chat-UniVi-V1.5 & 7B & 8 & 40.6 / 45.9 & 45.7 / 51.2 & 40.3 / 44.6 & 35.8 / 41.8 &   -  &   -  \\
	VideoLLaMA2     & 7B & 8 & 47.9 / \unk & 56.0 / \unk & 45.4 / \unk & 42.1 / \unk &   -  &   -  \\
	LLaVA-NeXT-QW2  & 7B & 8 & 49.5 / \unk & 58.0 / \unk & 47.0 / \unk & 43.4 / \unk &   -  &   -  \\
	\hdashline
	LongVILA        & 8B & 128 & 49.2 / \unk & 60.2 / \unk & 48.2 / \unk & 38.8 / \unk &   -  &   -  \\
	LongVA          & 7B & 128 & 52.6 / 54.3 & 61.1 / 61.6 & 50.4 / 53.6 & 46.2 / 47.6 &   -  &   -  \\ 
	Video-XL        & 7B & 128/256 & 55.5 / 61.0 & 64.0 / 67.4 & 53.2 / 60.7 & 49.2 / 54.9 & 64.9 &   -  \\
	LLaVA-OneVision$^*$ 
					& 7B & 32 & 58.2 / \unk & \unk / \unk & \unk / \unk & \unk / \unk & 64.7 & 56.3 \\
	\midrule
	VILA-V1.5       & 8B & 8 & 47.5 / 50.0 & 57.8 / 61.6 & 44.3	/ 46.2 & 40.3 / 42.1 & 46.3 & 47.1 \\
	\rule{0pt}{2ex} +\textsc{Frame-Voyager}\xspace
				    & 8B & 8 & 50.5 / 53.6 & 60.3 / 65.0 & 47.3 / 50.3 & 43.9 / 45.3 & 49.8 &   -  \\
	\rowcolor{gray!30}   
	\rule{0pt}{2ex} \textbf{+\mname}
	                & 8B & 8 & 53.3	/ 56.6 & 65.7 / 68.7 & 49.9 / 54.4 & 44.2 / 46.6 & 58.6 & 50.8 \\
	\hdashline
	MiniCPM-V2.6    & 7B & 8 & 52.6	/ 53.1 & 61.1 / 63.8 & 50.3	/ 50.2 & 46.4 / 45.4 & 55.4 & 51.2 \\
	\rowcolor{gray!30}
	\rule{0pt}{2ex} \textbf{+\mname} 
	                & 7B & 8 & 58.0 / 61.8 & 69.1 / 71.6 & 56.4 / 61.7 & 48.4 / 52.1 & 66.6 & 57.1 \\
	\hdashline
	LLaVA-OneVision & 7B & 8 & 53.6 / 53.9 & 64.7 / 65.0 & 51.4 / 52.0 & 44.7 / 44.6 & 59.3 & 54.2 \\
	\rule{0pt}{2ex} +\textsc{Frame-Voyager}\xspace
				    & 7B & 8 & 57.5 / \unk & 67.3 / \unk & 56.3 / \unk & 48.9 / \unk & 65.6 &   - \\
	\rowcolor{gray!30}	
	\rule{0pt}{2ex} \textbf{+\mname} 
	                & 7B & 8 & 59.6 / 59.1 & 71.8 / 71.1 & 57.6 / 57.6 & 49.6 / 48.8 & 69.8 & 59.0 \\
	\hdashline
	Ovis2 & 7B & 8 & 58.9 / 62.1 & 67.8 / 72.0 & 57.4 /60.3 & 51.3 / 54.0 & 60.9 & 56.9 \\
	\rowcolor{gray!30}
	\rule{0pt}{2ex} \textbf{+\mname} 
	                & 7B & 8 & \textbf{63.9 / 66.1} & \textbf{75.3 / 77.6} & \textbf{63.2 / 64.3} & \textbf{53.2 / 56.6} & \textbf{73.9} & \textbf{62.7} \\
	\bottomrule
	\end{tabular}
	\caption{
	Comparison of Video-LLMs with and without \mname{}.
	LVB$_{val}$ denotes the LongVideoBench validation set without subtitle interleaving.
	Video-XL uses 128/256 frames on Video-MME/MLVU.
	LLaVA-OneVision$^*$ refers to official report with well-tuned frames.
	}
	\label{tab:main_results}
\end{table*}

%% file: sec/4_experiment.tex
\section{Experiments}\label{sec:exp}

We integrate \mname{} as a plug-and-play process during inference with 7B SOTA Video-LLMs.
Following the setup in~\citet{yu2024frame}, we uniformly sample 128 candidate frames and apply \mname{} to select 8 frames as input to the Video-LLMs.
We conduct experiments using LMMs-Eval~\cite{lmms_eval2024} and VLMEvalKit~\cite{duan2024vlmevalkit} on three long video benchmarks.
Due to space limitations, additional experimental details and results are provided in the Appendix.
Code in \url{https://github.com/sunh-23/MDP3}.

\begin{figure*}[!t]
    \centering
    \begin{minipage}{0.64\textwidth}
     	\input{tables/long_video_bench_val}
     	\captionof{table}{
		    Performance of \mname{} and its absence on LongVideoBench$_{val}$ for varying video durations across different baseline Video-LLMs.
		}
		\label{tab:lvbv}
    \end{minipage}\hfill
    \begin{minipage}{0.34\textwidth}
    	\centering
	   	\includegraphics[width=1.0\linewidth]{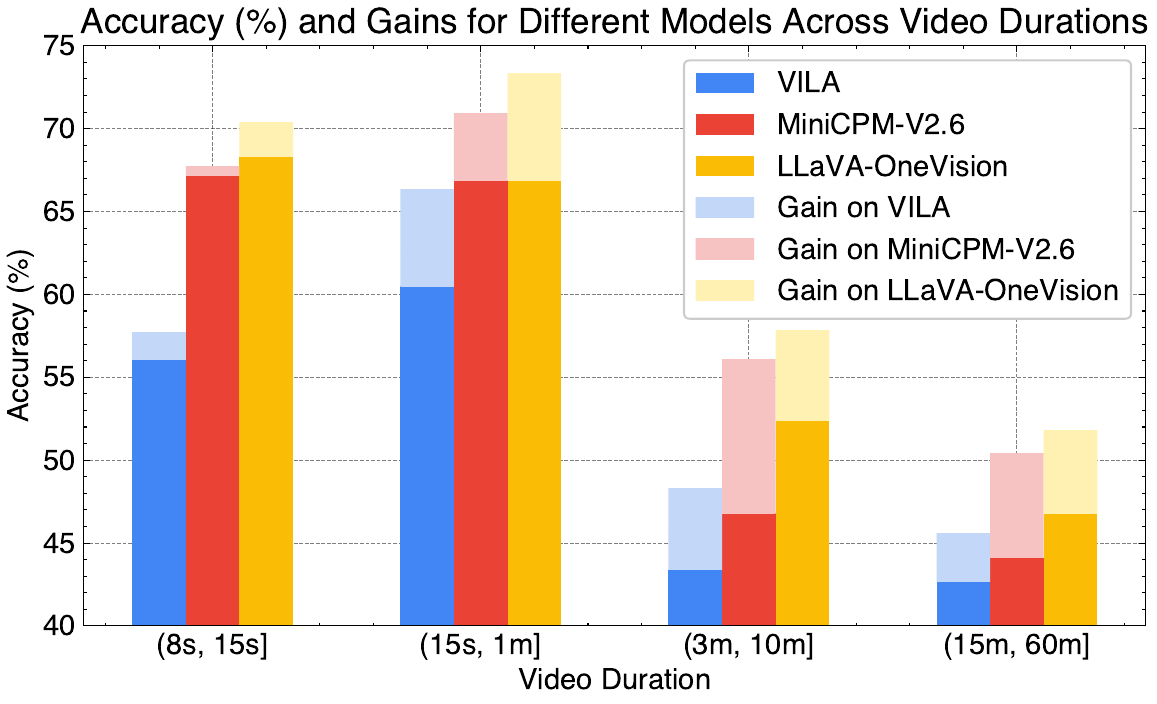}
		\captionof{figure}{
    		Performance of baselines and the improvement of \mname{} on LongVideoBench$_{val}$ for different video durations.
    		}
	\label{fig:lvbv}
    \end{minipage}
\end{figure*}

\subsection{Experiments Protocol}
\textbf{Backbone Models \& Primary Baselines:}
Based on the Video-MME leaderboard~\cite{fu2024video}, 
we choose the 7B SOTA Video-LLMs LLaVA-OneVision-7B~\cite{li2024llava}, MiniCPM-V2.6-7B~\cite{yao2024minicpm} and Ovis2-8B~\cite{ovis2} as our primary baselines.
To compare with the recent frame selection approach~\cite{yu2024frame}, tailor-made for VILA~\cite{lin2024vila}, we also include VILA-V1.5-7B as the primary baseline.
We integrate \mname{} with these baselines and evaluate it against other recent Video-LLMs.

\textbf{Benchmarks:}
We evaluate \mname{} on three long video benchmarks:
1)~\textbf{Video-MME}~\cite{fu2024video}, comprising 2700 human-annotated question-answer pairs, with an average video duration of 17 minutes. 
2)~\textbf{MLVU}~\cite{MLVU}, a comprehensive benchmark for multi-task long video understanding, consisting of 2593 tasks from 9 categories, with an average video duration of 12 minutes.
3)~\textbf{LongVideoBench}~\cite{wu2024longvideobench}, for which we use the validation set without subtitles, denoted as LVB$_{val}$, which contains 1337 question-answer pairs and has an average video length of 12 minutes.

\subsection{Results and Analysis}

\subsubsection{Comparison with SOTA Methods} 

\cref{tab:main_results} presents a comparison of \mname{} while applied to VILA-V1.5~\cite{lin2024vila},
MiniCPM-V2.6~\cite{yao2024minicpm}, LLaVA-OneVision~\cite{li2024llava} and Ovis2~\cite{ovis2}.
We compare \mname{} to recent Video-LLMs:
Video-LLaVA~\cite{lin2023video}, Qwen-VL-Chat~\cite{bai2023qwen}, VideoChat2~\cite{li2023videochat}, Chat-UniVi-V1.5~\cite{jin2024chat}, VideoLLaMA2~\cite{cheng2024videollama2}, LLaVA-NEXT-QW2~\cite{liu2024llavanext}, LongVILA~\cite{xue2024longvila}, LongVA~\cite{zhang2024longva}, and Video-XL~\cite{shu2024video}. 
The number of input frames for the LLMs is shown in \cref{tab:main_results}. 
Specifically, Video-XL processes 128 frames for Video-MME and 256 frames for MLVU, while LLaVA-OneVision$^*$ reports the official results using a well-tuned frame count.	
Compared to the baselines, \mname{} substantially improves the SOTA performance across all three evaluated benchmarks.
The main results presented in \cref{tab:main_results} highlight several key observations:
\begin{enumerate}
	\item Our frame selection method, \mname{}, empowers Ovis2 to achieve the highest overall performance across various benchmarks, confirming the effectiveness of \mname{}.
	\item When \mname{} is applied to the various primary Video-LLMs,
		it consistently outperforms uniform frame sampling.
		This validates that \mname{} is effective in frame selection and generalizes well across various Video-LLMs.
	\item Notably, LLaVA-OneVision with \mname{} selecting 8 frames outperforms LLaVA-OneVision$^*$, which uses a well-tuned frame count.
		This demonstrates that selecting key frames is more effective and efficient than simply increasing the number of frames, a finding supported by comparisons with LongVA, LongVILA, and Video-XL.
	\item Compared to \textsc{Frame-Voyager}, a custom-built method for VILA-V1.5 trained on extensive data, \mname{} shows superior performance. 
	This result validates that even a training-free, model-agnostic method can effectively achieve frame selection by reusing pretrained VLMs.
\end{enumerate}

\subsubsection{Results across Various Selection Sizes}\label{sssec:ravss}

\cref{fig:gain_on_videomme} presents the results of LLaVA-OneVision on Video-MME (without subtitles), comparing performance with and without applying \mname{} for key frame selection. 
There are four noteworthy insights:
\begin{enumerate}
	\item Regardless of whether \mname{} is applied, the law of diminishing marginal utility is evident as the number of frames fed into Video-LLMs increases.
		This shows that as the frame count rises, the visual understanding capacity of Video-LLMs quickly reaches its inherent upper bound, emphasizing the necessity for effective frame selection within limited capacity.
	\item When \mname{} is applied, the performance of LLaVA-OneVision consistently improves across various selection sizes, demonstrating the stability and generalizability of \mname{} under different selection size settings.
	\item The improvement of \mname{} is more pronounced when selecting fewer frames rather than more. 
		Notably, selecting 8 frames outperforms the best baseline of uniformly sampling 32 frames.
		Selecting just 2 frames (1/16 of the best baseline) achieves 90.8\% of the top performance, while 4 frames (1/8) reach 96.3\%.
		This shows that \mname{} effectively reduces the number of frames required for input, thereby enhancing inference efficiency.
	\item Notably, performance with 64 frames is lower than with 32 frames, whether or not \mname{} is used.
	This may be because LLaVA-OneVision was only trained on 32 frames, leading to potential performance drops when more frames are input.
	This indicates that current Video-LLMs often struggle to generalize beyond their training frame counts.
	Thus, effective frame selection methods like \mname{} are crucial for improving performance by focusing on frame selection during inference and avoiding the high costs of increasing frame counts during training.
\end{enumerate}

\subsubsection{Results across Various Video Durations}

We present the performance of baseline Video-LLMs and the improvements by \mname{} on LongVideoBench$_{val}$ for various video durations, as shown in \cref{tab:lvbv} and \cref{fig:lvbv}.
The results show that \mname{} substantially improves baseline performance for most video durations. 
However, for videos lasting 8 to 15 seconds, the improvements are smaller, and MiniCPM-V2.6 even shows a slight decline.
This is because uniformly sampling 8 frames in this range results in about 1 fps, making frame selection less impactful for very short videos.
This highlights that frame selection is more crucial for long videos than short ones. 
Indeed, for effectiveness and efficiency in video understanding, focusing on long videos provides a higher return on investment~(ROI).

\subsubsection{Ablation Study}

\input{tables/ablation_study.tex}

\begin{figure}[t]
	\centering
	\includegraphics[width=0.8\linewidth]{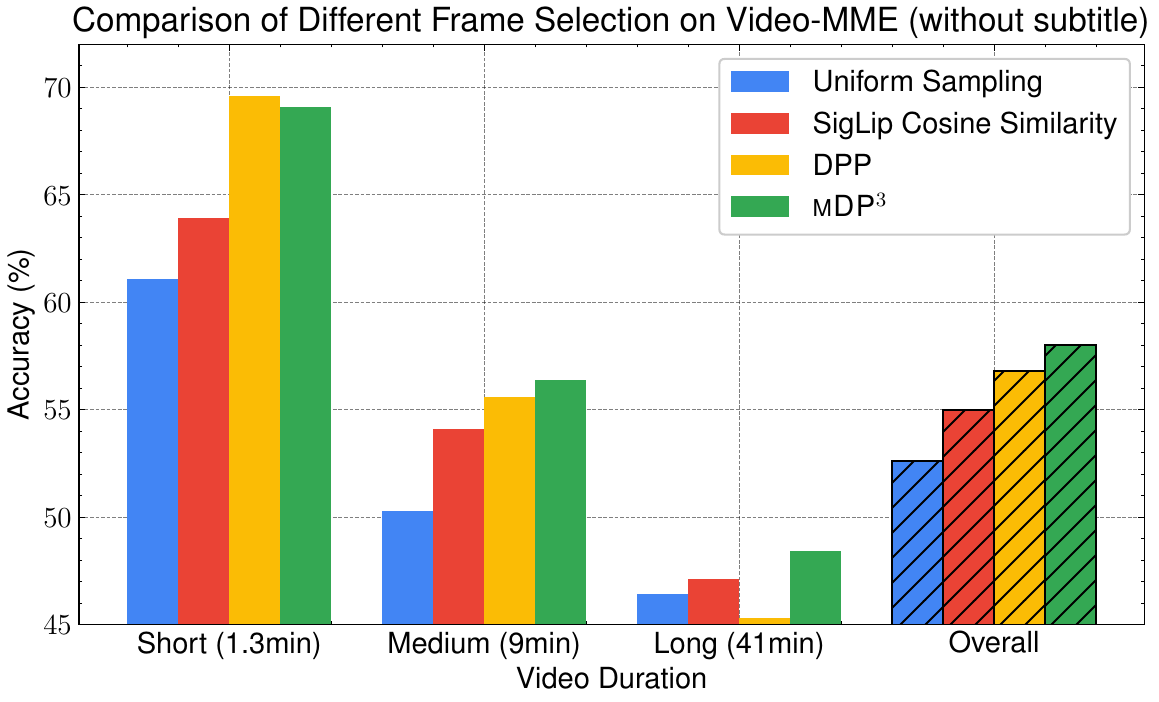}
	\caption{
	    Accuracies (\%) of various ablation variants of \mname{} applied with MiniCPM-V2.6 on Video-MME (without subtitles).
	}
   \label{fig:abl}
\end{figure}

We conducted an ablation study of \mname{}, with the results shown in \cref{tab:abl} and \cref{fig:abl}.
The study considered variations such as 
directly using SigLIP to select the top-$k$ similar frames based on the query, 
using standard DPP with CMGK while disregarding sequentiality,
and applying \mname{} with similarity matrix modifications such as MGK and cosine similarity.
We also compared \mname{} with traditional rule-based shot boundary detection methods,
such as RGB Histogram~\cite{sheena2015key}, Edges Change Ratio~\cite{nasreen2013key}, and Optical Flow~\cite{wolf1996key}, which select frames based on significant transitions between adjacent frames. 
These methods do not account for the current query or the list-wise relationship among selected frames.
The cluster-based method Katna~\cite{keplerlab_katna} only considers pair-wise diversity, overlooking query relevance, list-wise diversity, and sequentiality. 
\textsc{Frame-Voyager}~~\cite{yu2024frame} and \textsc{SeViLA}~~\cite{yu2023self} are recent end-to-end frame-selection methods designed for a specific model.
The \semicheckmark{} symbol indicates that some factors, while not explicitly addressed, may be implicitly captured.
Training on VILA's hidden states with a list-based loss enables \textsc{Frame-Voyager} to potentially capture list-wise information, while position embeddings may address sequentiality.

\cref{tab:abl} shows that the key principles of query relevance, list-wise diversity, and sequentiality are crucial for frame selection.
The complete \mname{} is the first to incorporate all these principles and achieves the best performance in the ablation study.
A comparison of the ablation variants of \mname{} further confirms that 
all components are effective and indispensable, including high-dimensional similarity in RKHS, list-wise diversity with DPP, and sequentiality with MDP.
\cref{fig:abl} further illustrates the ablation results across various video durations. 
While \mname{} outperforms other methods overall, it performs slightly worse than DPP for short videos.
Standard DPP accounts for list-wise diversity but lacks sequentiality.
However, the performance of standard DPP deteriorates significantly with longer videos.
These observations reveal two conclusions:
\begin{enumerate}
    \item Sequentiality is more critical for long videos than for short ones, and it should be prioritized when selecting frames from longer videos.
    \item Standard DPP cannot handle sequentiality, which becomes a drawback as video length increases.  Therefore, sequentiality is essential for frame selection and \mname{}.
\end{enumerate}

%% file: tables/long_video_bench_val.tex
\setlength\tabcolsep{3.5pt}
\begin{tabular}{lcccccc}
    \toprule
		\multirow{2}{*}{\textbf{Model}} & 
		\multicolumn{5}{c}{\textbf{LongVideoBench$_{val}$}} \\

		\cmidrule(lr){2-6}
		\quad\ \textit{Video Duration} & 
			Overall  & 
			(8s, 15s] & 
			(15s, 1m] & 
			(3m, 10m] & 
			(15m, 60m] \\ 
  	\toprule
	VILA-V1.5       & 47.1\emptyfoot{+3.7} & 56.1\emptyfoot{+1.6} & 60.5\emptyfoot{+5.8} & 43.4\emptyfoot{+4.9} & 42.7\emptyfoot{+2.9} \\
	\rowcolor{gray!30}
	\rule{0pt}{2ex} \textbf{+\mname}
	                & 50.8\gainfoot{+3.7} & 57.7\gainfoot{+1.6} & 66.3\gainfoot{+5.8} & 48.3\gainfoot{+4.9} & 45.6\gainfoot{+2.9} \\
	\hdashline
	MiniCPM-V2.6    & 51.2\emptyfoot{+5.9} & 67.7\emptyfoot{-0.5} & 66.9\emptyfoot{+4.0} & 46.8\emptyfoot{+9.3} & 44.1\emptyfoot{+6.3} \\
	\rowcolor{gray!30}
	\rule{0pt}{2ex} \textbf{+\mname} 
	                & 57.1\gainfoot{+5.9} & 67.2\lossfoot{-0.5} & 70.9\gainfoot{+4.0} & 56.1\gainfoot{+9.3} & 50.4\gainfoot{+6.3} \\
	\hdashline
	LLaVA-OneVision & 54.2\emptyfoot{+4.8} & 68.3\emptyfoot{+2.1} & 66.9\emptyfoot{+6.4} & 52.4\emptyfoot{+5.4} & 46.8\emptyfoot{+5.0} \\
	\rowcolor{gray!30}
	\rule{0pt}{2ex} \textbf{+\mname}
	                & 59.0\gainfoot{+4.8} & 70.4\gainfoot{+2.1} & 73.3\gainfoot{+6.4} & 57.8\gainfoot{+5.4} & 51.8\gainfoot{+5.0} \\
	\bottomrule
	\end{tabular}

%% file: tables/ablation_study.tex
\begin{table}[t]
	\centering
	\setlength\tabcolsep{2.0pt}
	\begin{tabular}{lccccccc}
		\toprule
	                  & \textit{Qry.} & \textit{Div.} & \textit{Seq.} & \textit{HDS} & \textit{TFMA} & Acc. \\ 
		\midrule
	VILA-8B (+ Uniform)    & \nocheckmark  & \nocheckmark & \nocheckmark  & \nocheckmark & \checkmark & 47.5 \\
	\ + RGB Histogram      & \nocheckmark  & \nocheckmark & \nocheckmark  & \nocheckmark & \checkmark & 45.9 \\
	\ + Edges Change       & \nocheckmark  & \nocheckmark & \nocheckmark  & \nocheckmark & \checkmark & 47.3 \\
	\ + Optical Flow       & \nocheckmark  & \nocheckmark & \nocheckmark  & \nocheckmark & \checkmark & 46.7 \\
	\ + Katna              & \nocheckmark  & \nocheckmark & \nocheckmark  & \nocheckmark & \checkmark & 45.7 \\
	\ + \textsc{SeViLA}    
						  & \semicheckmark & \semicheckmark & \semicheckmark  & \nocheckmark & \nocheckmark &  49.0 \\
	\ + \textsc{Frame-Voyager}           
						   & \checkmark    & \semicheckmark  & \semicheckmark & \nocheckmark & \nocheckmark & 50.5 \\
		\hdashline
	\ + SigLIP 	           & \checkmark    & \nocheckmark & \nocheckmark  & \nocheckmark & \checkmark & 50.6 \\
	\ + \mname{} w. MGK    & \nocheckmark  & \checkmark   & \checkmark  & \checkmark & \checkmark & 48.9 \\
	\ + DPP w. CMGK  	   & \checkmark    & \checkmark   & \nocheckmark  & \checkmark & \checkmark & 51.8 \\
	\ + \mname{} w. cos. sim. & \checkmark & \checkmark   & \checkmark  & \nocheckmark & \checkmark & 50.2 \\
		\rowcolor{gray!30}
	\ \textbf{+ \mname}    & \checkmark    & \checkmark   & \checkmark    & \checkmark & \checkmark & \textbf{53.3} \\	
		\bottomrule
	\end{tabular}
	\caption{
		Results (\%) of various frame selection methods and ablation variants of \mname{} on Video-MME ($wo\ subs$), based on VILA. 
		Indicators: query relevance~(\textit{Qry}), list-wise diversity~(\textit{Div}), list-wise sequentiality~(\textit{Seq}), high-dimensional RKHS similarity (\textit{HDS}), and a tuning-free, model-agnostic design~(\textit{TFMA}).
	}
	\label{tab:abl}
\end{table}

%% file: sec/5_conclusion.tex
\section{Conclusion}\label{sec:conclusion}

This paper emphasizes that good frame selection should consider query relevance, list-wise diversity, and sequentiality.
We propose \mname{}, a training-free and model-agnostic method that fully addresses these principles for the first time.
Experiments on three long video benchmarks show that integrating \mname{} with existing Video-LLMs significantly improves performance, demonstrating its effectiveness and robustness.
\mname{} also provides a $(1-1/e)$-approximation for the NP-hard list-wise frame selection and selection size allocation problem in sequential video segments.
This method has a practical time complexity of \(\mathcal{O}(nk^4)\), reducible to \(\mathcal{O}(nk^2)\) through lazy and parallel updates.
Future work could leverage training-free \mname{} to perform test-time computation scaling up, generating self-distilled data for iterative Video-LMMs refinement.

%% file: sec/X_suppl.tex
\clearpage
\appendix
\setcounter{page}{1}
\maketitlesupplementary

\section*{Appendix Contents}

\startcontents[appendices]
\printcontents[appendices]{}{1}{}

\newpage
\input{appendix/proof.tex}

\newpage
\input{appendix/experiments_details.tex}

\newpage
\section{Limitations of \mname{} and Future Directions}
\mname{} is a training-free, model-agnostic method that, for the first time, fully addresses query relevance, list-wise diversity, and sequentiality.
Theoretically, \mname{} offers a \((1 - 1/e)\)-approximate solution to the NP-hard list-wise frame selection problem, achieving pseudo-polynomial time complexity and demonstrating its efficiency.
Empirically, \mname{} outperforms existing methods significantly, confirming its effectiveness and robustness.

However, \mname{} still has certain limitations that merit further exploration in future research.
\begin{enumerate}
    \item 
    	\textbf{Limitation:} The use of pretrained VLMs to develop a training-free, model-agnostic method is a double-edged sword. 
    	While \mname{} can be seamlessly integrated into existing Video-LLMs, pretrained VLMs often have limitations in understanding complex instructions.
    	\textbf{Future Directions:} Fortunately, \mname{} is highly adaptable for future extensions. 
    	Fine-tuning the VLMs within \mname{} with more complex instructions could significantly improve frame selection.
    	Specifically, although the selection process is discrete and not directly optimizable, paired selection data can be gathered, and contrastive learning methods (such as DPO) can be applied for fine-tuning.
    	The selection order could be supervised using existing LLMs, with the list-wise score finetuned to align with this supervision.
    \item 
    	\textbf{Limitation:} The selection size $k$ is fixed in \mname{}, and as shown in the case study, \mname{} may occasionally select some useless frames.
    	This issue can be mitigated by adjusting the trade-off between relevance and diversity, but such strategies are not feasible to apply on each sample.
		\textbf{Future Directions:} Therefore, exploring how to set an adaptive selection size $k$ is a promising area for future research. In \mname{}, during dynamic programming, the optimal selection for any size $i < k$ has been captured in the trace matrix $\mathcal{T}_{T,i}$. This provides a convenient framework for determining the optimal $k$, but the challenge of identifying the best $i < k$ still remains and warrants further investigation.
\end{enumerate}

%% file: appendix/proof.tex
\section{Proofs}

\subsection{Preliminaries}

\begin{definition}[Submodular Maximization~\cite{lovasz1983submodular}]\label{thm:sub_def}
	Let $\Omega$ denote a finite set, and let $f: 2^{\Omega} \rightarrow \mathbb{R}_{\geq 0}$ be a set function, where $2^{\Omega}$ is the power set of $\Omega$. The function $f$ is called submodular if it satisfies one of the following three equivalent conditions:
	\begin{itemize}
		\item For every $X, Y \subseteq \Omega$ with $X \subseteq Y$, and for all $x \in \Omega \setminus Y$, we have 
			\begin{equation}
				f(X \cup \{x\}) - f(X) \geq f(Y \cup \{x\}) - f(Y) \,,
			\end{equation}
		\item For every $X, Y \subseteq \Omega$, we have 
			\begin{equation}\label{thm:sub_def_cond2}
				f(X) + f(Y) \geq f(X \cup Y) + f(X \cap Y) \,,
			\end{equation}
		\item For every $X \subseteq \Omega$ and $x_1, x_2 \in \Omega \setminus X$ such that $x_1 \neq x_2$, we have 
			\begin{equation}
				f(X \cup \{x_1\}) + f(X \cup \{x_2\}) \geq f(X \cup \{x_1, x_2\}) + f(X) \,.
			\end{equation}
	\end{itemize}
\end{definition}
These three conditions are equivalent, and the first condition is the most commonly used, as it directly reflects the law of diminishing marginal utility as the number of items increases.

\begin{definition}[Sub-additive]
	A set function $f: 2^{\Omega} \rightarrow \mathbb{R}_{\geq 0}$ is sub-additive if for every two sets $X, Y \in \Omega$, we have \
	\begin{equation}
		f(X\cup Y)\leq f(X) + f(Y)\,.
	\end{equation}
\end{definition}

\begin{lemma}\label{thm:sub_additive}
	A non-negative submodular set function $f: 2^{\Omega}\rightarrow\mathbb{R}_{\geq 0}$ is sub-additive.
	\begin{proof}
		As the second condition \cref{thm:sub_def_cond2} in \cref{thm:sub_def}. 
		for $X,Y\subseteq \Omega$, we have $f(X) + f(Y) \geq f(X \cup Y) + f(X \cap Y)$. So, $f(X) + f(Y) \geq f(X \cup Y)$ as $f(X \cap Y) \geq 0$.
	\end{proof}
\end{lemma}

\begin{lemma}\label{thm:fs_submodular}
	Let $f: 2^{\Omega}\rightarrow \mathbb{R}$ be submodular. 
	Let $S\subseteq \Omega$, and $f_{S}(X)=f(S\cup X)-f(S)$ for every $X\subseteq\Omega$. ($f_{S}$ is the marginal value function for set $S$.) 
	Then $f_S$ is also submodular.
	\begin{proof}
		Let $X, Y\subseteq\Omega \setminus S$; it suffices to consider ground set $\Omega\setminus S$.
		\begin{align}
			&\left(f_S(X\cup Y)+f_S(X\cap Y)\right)-\left(f_S(X)-f_S(Y)\right) \notag\\
			&\quad=f(S\cup X\cup Y)-f(S)+f(S\cup(X\cap Y))-f(S) \notag\\
			&\qquad-\left(f(S\cup X)-f(S)+f(S\cup Y)-f(S)\right) \\
			&\quad = f(S\cup X\cup Y)+f(S\cup(X\cap Y)) \notag\\
			&\qquad-f(S\cup X)-f(S\cup Y) \notag\\
			&\quad \leq 0 \,. \notag
		\end{align}
		The last inequality is by $S\cup X\cup Y=(S\cup X)\cup(S\cup Y)$, $S\cup(X\cap Y)=(S\cup X)\cap(S\cup Y)$ and submodularity of $f$.
		Therefore, $f_S$ is also submodular is proved.
	\end{proof}
\end{lemma}

\subsection{Proof of $(1-1/e)$-approximation}

Submodular maximization is NP-hard in general.
Therefore, most research in this field focuses on approximation algorithms with polynomial-time complexity. 
While the submodular function is monotone, \ie, for every $X, Y \subseteq \Omega$, we have $f(X) \leq f(Y)$. 
The problem of maximizing a monotone submodular function subject to a cardinality constraint admits a $(1 - 1/e)$-approximation greedy algorithm~(as introduced  in \cref{subsec:method_dpp})~\cite{nemhauser1978analysis}.

In this section, we provide a concise proof of the $1 - 1/e$ approximation ratio for the greedy algorithm.
\begin{theorem}[$(1 - \frac{1}{e})$-approximation of Greedy Algo.]
    There exists a greedy algorithm for the submodular maximization problem, which starts with an empty set \( S = \emptyset \) and iteratively selects the item that maximizes the marginal gain:
    \begin{equation}\label{eq:greed_sub}
        j = \mathop{\mathrm{argmax}}\limits_{i \in \Omega \setminus S} \quad f(S \cup \{i\}) - f(S) \,.
    \end{equation}
    The algorithm continues until the selected set \( S \) reaches the cardinality limit \( k \).
    
    This greedy algorithm provides a solution \( \hat{S} \), which guarantees a $(1 - 1/e)$ approximation, where the optimal solution is denoted as \( S^* \):
    \begin{equation}
        f(\hat{S}) \leq f(S^*)
    \end{equation}
\end{theorem}
\begin{proof}
    According to \cref{thm:fs_submodular}, \(f_S\) is submodular, and by \cref{thm:sub_additive}, it is also sub-additive. Therefore, we have:
    \begin{equation}
        f_S(S^*) \leq \sum_{x \in S^*} f_S(x)\,,
    \end{equation}
    which implies that:
    \begin{equation}
        \exists x \in S^*, \quad f_S(x) \geq \frac{1}{k} f_S(S^*) \,.
    \end{equation}
    For this \(x\), we have the following margin lower bound:
    \begin{equation}\label{eq:margin_lowerbound}
        f(S \cup \{x\}) - f(S) \geq \frac{f(S^*) - f(S)}{k}\,.
    \end{equation}
    
    Let \( S_t \) denote the selected subset after the \(t\)-th step of the greedy algorithm. According to \cref{eq:greed_sub}, in the greedy algorithm, we have:
    \begin{equation}
        f(S_{t+1}) - f(S_t) \geq f(S_t \cup \{x\}) - f(S_t), \quad \forall x \in \Omega \setminus S_t\,.
    \end{equation}
    Therefore, in the greedy algorithm, the marginal gain is lower-bounded by:
    \begin{equation}\begin{aligned}\label{eq:margin_gain_lowerbound}
        f(S_{t+1}) - f(S_t) &\geq \frac{f(S^*) - f(S_t)}{k}\,.
    \end{aligned}\end{equation}
    This implies:
    \begin{equation}
        f(S^*) - f(S_{t+1}) \leq \left(1 - \frac{1}{k}\right)(f(S^*) - f(S_t))\,.
    \end{equation}
    Hence, when the greedy algorithm selects a subset \( \hat{S} = S_k \) after \(k\) steps, we have:
    \begin{equation}\begin{aligned}
        f(S^*) - f(S_k) &\leq \left(1 - \frac{1}{k}\right)(f(S^*) - f(S_{k-1})) \\
        &\leq \left(1 - \frac{1}{k}\right)^2(f(S^*) - f(S_{k-2})) \\
        & \qquad \vdots \\
        &\leq \left(1 - \frac{1}{k}\right)^k(f(S^*) - f(S_0))\,,
    \end{aligned}\end{equation}
    where \( S_0 \) is the initial set at \( t = 0 \), with \( S_0 = \emptyset \), such that \( f(S_0) = 0 \).
    Therefore, we obtain:
    \begin{equation}
        f(S^*) - f(\hat{S}) \leq \left(1 - \frac{1}{k}\right)^k f(S^*)\,.
    \end{equation}
    Hence, we have:
    \begin{equation}\begin{aligned}\label{eq:detailed_bound}
        f(\hat{S}) 
            &\geq \left( 1 - \left(1 - \frac{1}{k}\right)^k \right) f(S^*) \\
            &\geq \lim_{k \to +\infty} \left( 1 - \left(1 - \frac{1}{k}\right)^k \right) f(S^*) \\
            & = \left(1 - \frac{1}{e}\right) f(S^*)\,.
    \end{aligned}\end{equation}
    The proof is complete.
\end{proof}

Previous works~\cite{fujishige2005submodular,krause2014submodular} have established that the determinantal point process~(DPP) is a monotone submodular function.
Therefore, when selecting 8 frames using the standard DPP, the approximation ratio is at least \mbox{\( 1 - \left( 1 - \frac{1}{8} \right)^8 = 65.6\% \).}

\subsection{Proof of Time Complexity}
The time complexity of \cref{alg} is closely $\mathcal{O}(nk^3)$, where $n$ represents the number of candidate frames and $k$ denotes the selection size.

\begin{proof}
    As indicated by the pseudocode in \cref{alg}, there are three loops, with the iteration sizes specified as follows:
    \begin{itemize}
        \item Line 3: $\lceil \frac{n}{m} \rceil$;
        \item Line 6: $k + 1$;
        \item Line 7: $k_{t-1} \leq \min(m, k, k - C_{t-1}) \leq \min(m, k)$.
    \end{itemize}
    Here, $m$ denotes the segment size. 
    The DPP update in each iteration (Line 9) has a time complexity of $\mathcal{O}(mk_{t-1})$, utilizing Cholesky decomposition for incremental computation.

    Therefore, the overall time complexity can thus be expressed as:
    \begin{equation}
        \mathcal{O}\left(
            \frac{n}{m} \cdot (k \cdot mk_{t-1}^2 ) 
        \right) = \mathcal{O}\left(
            n \cdot \min(m^2k, k^3)  
        \right)
    \end{equation}
    Additionally, the time complexity of computing the determinant in line 5 is $\mathcal{O}(m^3)$, resulting in the following time complexity:
    \begin{equation}
        \mathcal{O}\left(
            \frac{n}{m} \cdot m^3
        \right) = \mathcal{O}\left(
            n m^2
        \right)
    \end{equation}
    Next, we analyze the relationship between $m$ and $k$, leading to the following total time complexity:
    \begin{equation}
        \left\{
            \begin{array}{ll}
                \mathcal{O}(nm^2),  & \text{if } k < m^{2/3} \\
                \mathcal{O}(nk^3),  & \text{if } m^{2/3} \leq k < m  \\
                \mathcal{O}(nm^2k) < \mathcal{O}(nk^3), & \text{if } k \geq m \\
            \end{array}
        \right.
    \end{equation}
    In practice, the segment size $m$ is typically small, so in most practical cases, the time complexity of \cref{alg} is upper-bounded by $\mathcal{O}(nk^3)$.
\end{proof}
This algorithm exhibits pseudo-polynomial time complexity, analogous to the knapsack problem. 
In the dynamic programming approach proposed in this paper, the runtime of the pseudo-polynomial complexity is practically comparable to that of a polynomial-time algorithm.

In practical applications, the segment size $m$ and selection size $k$ are both much smaller than the total number of video frames $n$, \ie, $m \ll n$ and $k \ll n$. 
Thus, regardless of the relationship between $m$ and $k$, the total time complexity is much smaller than feeding all frames into transformer-based LLMs, which have a time complexity of $O((n \cdot \text{\#tokens\_per\_image})^2)$ per layer and attention head.

Additionally, the iteration in line 6 of \cref{alg} is independent, enabling parallel updates. 
This results in a time complexity of \(\mathcal{O}(nk^2)\), making its efficiency comparable to that of the standard DPP. 
Moreover, without lazy strategies and parallel updates, the time complexity is closely \(\mathcal{O}(nk^4)\) in most practical cases, and the proof is similar.

%% file: appendix/experiments_details.tex
\section{Experiments Details and More Discussions}

\subsection{Experiments Compute Resoureces}
We integrate \mname{} as a plug-and-play process during inference with 7B SOTA Video-LLMs.
We conduct experiments using LMMs-Eval~\cite{lmms_eval2024} and VLMEvalKit~\cite{duan2024vlmevalkit} on three long video benchmarks. 
All experiments are run on NVIDIA A100-PCIE-40GB or NVIDIA A100-PCIE-80GB GPUs, using 256 AMD EPYC 7H12 64-Core @ 2.600GHz CPUs, with Ubuntu 20.04.6 as the operating system, adhering to a rigorous experimental protocol to ensure fair comparisons among compared methods.

\subsection{Dataset Details}

\paragraph{Video-MME~\cite{fu2024video}:}  
Video Multi-Modal Evaluation~(Video-MME) is a dataset designed to enhance video understanding for Multimodal Large Language Models (MLLMs). It consists of 900 videos spanning 6 visual domains, with durations ranging from 11 seconds to 1 hour, capturing a variety of contextual dynamics.
All videos are manually annotated by experts, generating 2700 question-answer pairs, ensuring high-quality and reliable data for model evaluation.
Experiments on Video-MME will be conducted both with and without subtitles to assess the impact of multi-modal inputs.

\paragraph{MLVU~\cite{MLVU}:}
Multi-task Long Video Understanding Benchmark~(MLVU) is a new dataset designed to evaluate Long Video Understanding (LVU) performance.
It addresses the limitations of existing benchmarks by offering longer video durations, diverse video genres (such as movies, surveillance footage, and cartoons), and a range of evaluation tasks. 
The benchmark includes 2593 tasks across 9 categories, with an average video duration of 12 minutes, providing a comprehensive assessment of MLLMs' capabilities in understanding long videos.
This allows for a more comprehensive assessment of MLLMs' capabilities in understanding long videos.

\paragraph{LongVideoBench~\cite{wu2024longvideobench}:}  
It is a recent benchmark designed to evaluate long-term video-language understanding for MLLMs.
It consists of 3763 web-collected videos of varying lengths, up to one hour, with subtitles, covering a wide range of themes.
The dataset is tailored to assess models' ability to process and reason over detailed multimodal information from long video inputs.
It includes 6678 human-annotated multiple-choice questions across 17 fine-grained categories, making it one of the most comprehensive benchmarks for long-form video understanding.
In this paper, we focus on the validation set without subtitles, denoted as LVB$_{val}$, which contains 1337 question-answer pairs and has an average video length of 12 minutes.

\subsection{Implementation Details}

\paragraph{Visual Encoder:}
The primary baselines, LLaVA-OneVision-7B~\cite{li2024llava} and MiniCPM-V2.6-7B~\cite{yao2024minicpm}, both use SigLIP~\cite{zhai2023sigmoid} as the visual encoder,
which we also adopt as the VLM in~\cref{eq:emb} to avoid including excessive additional parameters. 
Additionally, since the context length of the text encoder in SigLIP is 64, potentially insufficient for the entire question, 
we split the text sequence into multiple sequences of equal length, each no longer than 64 tokens.
We then extract multiple text embeddings and aggregate them into a final text embedding using pooling.

\paragraph{Multi-kernel:}
\cref{eq:k_family} presents a principled approach for selecting the optimal kernel.
Factors exist to combine positive semi-definite (PSD) kernels.
For implementation, we use the Gaussian kernel as the base PSD kernel, defined by 
\(k(x,y) = \exp\left(-\frac{\| (x-y)/h \|^2}{2\sigma^2}\right)\). 
Consequently, the base kernel with a combination factor can be expressed as:
\begin{equation}
    \beta_u \cdot k_u(x,y) = \beta_u \cdot \exp\left(-\frac{\| x-y \|^2}{2(h_u\sigma_u)^2}\right) \,.
\end{equation}
We denote $\alpha_u = (h_u\sigma_u)^2$ as a single hyperparameter.
Following the multi-kernel maximum mean discrepancy~(MK-MMD) framework~\cite{long2015learning}, the optimal $\beta_u$ can be optimized using a quadratic program~(QP).
However, optimizing $\beta_u$ is orthogonal of this work.
In the recent official implementation of \citet{long2015learning}, average weights $\beta_u=1/U$ were employed, yielding good performance. 
Hence, we adopt average weights for $\beta_u$ and concentrate on configuring $\alpha_u$. 
Consistent with \citet{long2015learning,sun2023enhancing}, for both $g,k \in \mathcal{K}$, we set $\alpha_u$ to $2^i$ where $i \in \{-3, -2, 0, 1, 2\}$ and use an averaged ensemble of multiple Gaussian kernels. 
Besides, the difference between $g$ and $k$ is modulated by $\lambda$.

\paragraph{Hyper-parameters:}
We set the trade-off hyperparameter $\lambda = 0.2$ and the segment size $m = 32$ for all tasks and benchmarks in \mname{}, determined through cross-validation on a subset of the LLaVA-OneVision-mid~\cite{li2024llava} training set.

\subsection{Hyperparameter Sensitivity Analysis}
\begin{figure}[t]
	\centering
	\includegraphics[width=0.8\linewidth]{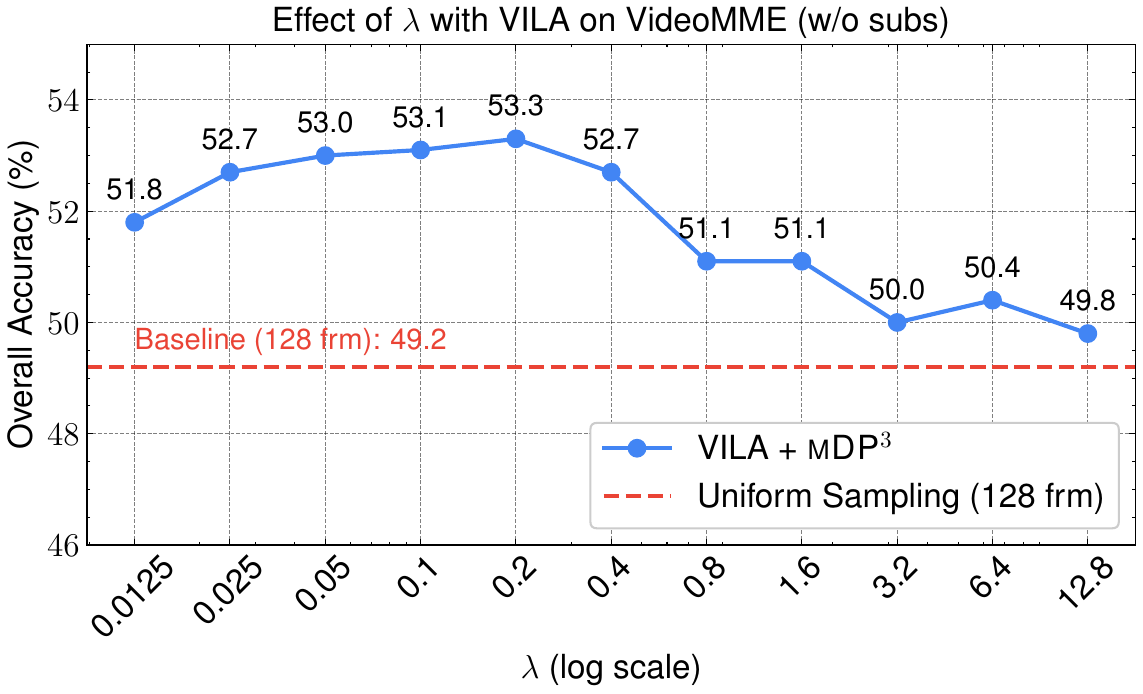}
	\caption{\label{fig:hyper_sen_lambda}
 	   Hyper-parameter sensitivity with trade-off $\lambda$.
	} 
\end{figure}

\begin{figure}[t]
	\centering
	\includegraphics[width=0.8\linewidth]{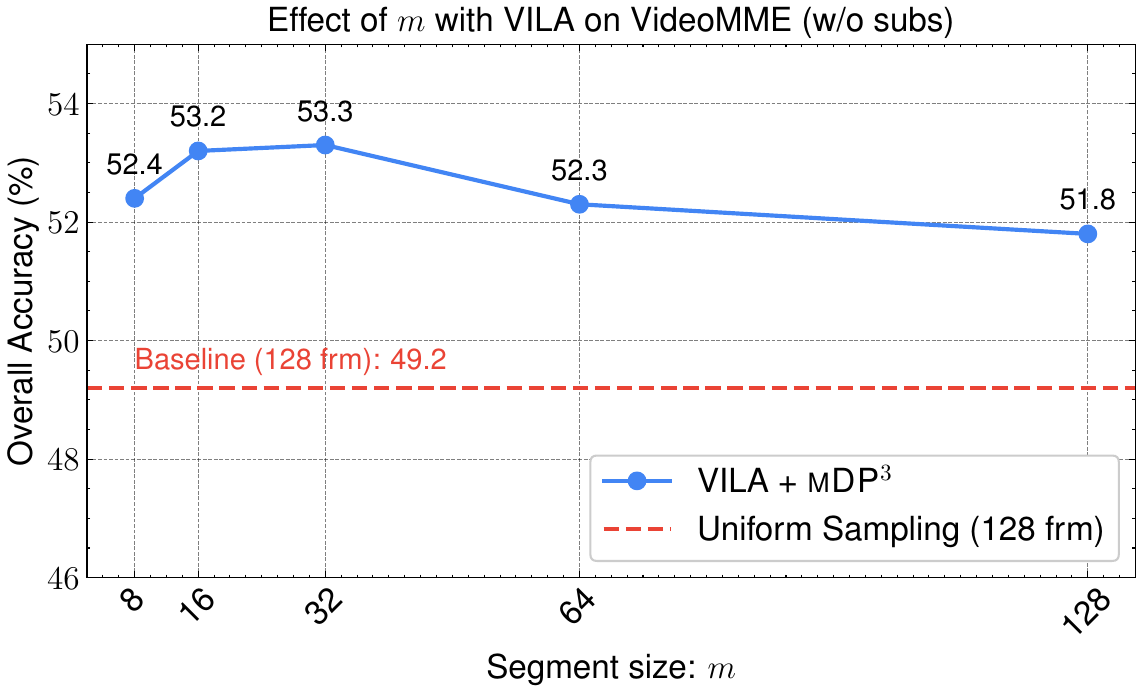}
	\caption{\label{fig:hyper_sen_m}
		Hyper-parameter sensitivity with segment size $m$.
	}
\end{figure}

To evaluate the robustness of \mname{}, we perform a sensitivity analysis of the score-trade-off parameter $\lambda$ and the segment size $m$ using the VideoMME benchmark (without subtitles). The results are shown in \cref{fig:hyper_sen_lambda,fig:hyper_sen_m}.

\paragraph{Score-trade-off ($\boldsymbol{\lambda}$).}
\Cref{fig:hyper_sen_lambda} varies $\lambda$ logarithmically from 0.0125 to 12.8. Accuracy stays within a narrow range of 51–53 \%, peaking at 53.3 \% when $\lambda = 0.2$; every value surpasses the baseline of 49.2 \%.

\paragraph{Segment size ($\boldsymbol{m}$).}
\Cref{fig:hyper_sen_m} evaluates segment sizes ranging from 8 to 128 frames. All configurations exceed the baseline. The best accuracy, 53.3 \%, is achieved at $m = 32$; accuracy stays above 52 \% for $m \le 64$ and remains competitive (51.8 \%) even at $m = 128$. These findings indicate that \mname{} is insensitive to the choice of $m$.

Across all examined values of $\lambda$ and $m$, \mname{} consistently outperforms uniform sampling, confirming its robustness and practical ease of tuning.

\subsection{Additional Parameters Analysis}

\begin{table*}[!t]
	\centering
	\begin{tabular}{lccccc}
		\toprule
			MLLM & Used LLM & Used VLM & MLLM Params & Additional Params & Increase \\
		\midrule
			VILA-V1.5-8B       & LLama3-8B & SigLIP-400M & 8.494B & 0.450B & 5.298\% \\
			MiniCPM-V2.6-7B    & Qwen2-7B  & SigLIP-400M & 8.099B & 0.450B & 5.556\% \\
			LLaVA-OneVision-7B & Qwen2-7B  & SigLIP-400M & 8.027B & 0.450B & 5.606\% \\
		\bottomrule
	\end{tabular}
	\caption{
    	Parameter scales for VILA-V1.5-8B, MiniCPM-V2.6-7B, and LLaVA-OneVision-7B, along with the increase due to the additional parameters introduced by \mname{}.
	    Here, ``MLLM Params'' refers to the parameter scale of the baseline, including the LLM, the visual encoder in the VLMs, and the projector between them.
	    ``Additional Params'' comes from the text encoder of the pretrained SigLIP, introduced by \mname{}.
	}
	\label{tab:add_param}
\end{table*}

\mname{} is a training-free, model-agnostic method that leverages pretrained VLMs.
The primary baselines,
VILA-V1.5-8B~\cite{lin2024vila}, MiniCPM-V2.6-7B~\cite{yao2024minicpm}, and LLaVA-OneVision-7B~\cite{li2024llava},
all integrate the vision encoder from SigLIP~\cite{zhai2023sigmoid}, so \mname{} only needs to introduce the additional parameter from the text encoder in SigLIP.
The parameter scales of them are reported in \cref{tab:add_param}.
The results indicate that the additional parameters from the text encoder amount to no more than 6\% of the original MLLM scale, which is negligible.
More importantly, these parameters are pretrained in VLMs and do not require tuning with the specific MLLMs.

\subsection{Latency Comparison}
 
\begin{figure}[t]
	\centering
	\includegraphics[width=\linewidth]{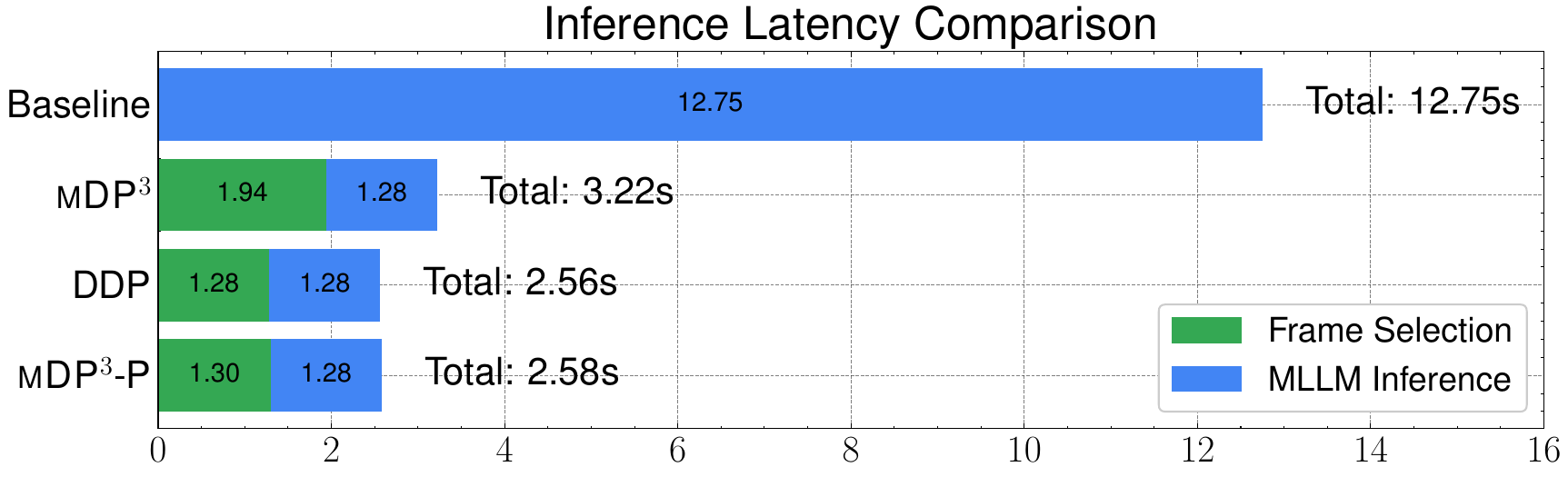}
	\caption{  
 	   Latency during inference with MiniCPM-V2.6 on Video-MME (without subtitles). 
 	   Where \mname{}-P refers o the \mname{} with parallel computation in dynamic programming~(\ie line 9 in \cref{alg}).
	   For direct comparison and better visualiztion, we omit the latency of identical processes across all compared models, including image loading and processing in the visual processor and encoder.  
	}
   \label{fig:latency}
\end{figure}

Our method, \mname{}, is training-free, which eliminates any additional latency during the training phase.
Accordingly, we report the average latency of various processes during inference with MiniCPM-V2.6 on Video-MME~(without subtitles), as illustrated in \cref{fig:latency}.
The baseline refers to MiniCPM-V2.6 processing all 128 candidate frames without applying frame selection.
%
Compared to the baseline, \mname{} selects only 8 essential frames as input to the MLLMs, achieving a significant speedup during inference.  
It reduces MLLM inference time by 11.47s (90\% of the baseline) while requiring only an additional 1.94s (15\% of the baseline) for frame selection.
\mname{}-P accelerates frame selection using parallel computation in dynamic programming~(\ie line 9 in \cref{alg}), requiring only an additional 1.30s for frame selection while reducing MLLM inference time by 11.47s (90\% of the baseline).
Furthermore, \mname{}-P has a latency comparable to the basic DPP method, which accounts only for list-wise diversity.  
This accords to our theoretical analysis, as both \mname{}-P and DPP share the same time complexity of $\mathcal{O}(nk^2)$.

\subsection{Additional Experiments about Ablation Study}

We conduct detailed ablation experiments on three Video-LLMs~(VILA-V1.5, MiniCPM-V2.6, LLaVA-OneVision) and three benchmarks~(Video-MME, MLVU, LVB$_{val}$), demonstrating the consistent superiority of \textsc{mDP}$^3$ over alternative strategies. 

\textsc{mDP}$^3$ attains peak scores across all configurations: 53.3/58.6/50.8 for VILA-V1.5, 58.0/66.6/57.1 for MiniCPM-V2.6, and 59.6/69.8/59.0 for LLaVA-OneVision.
The most significant improvements are observed on MLVU, with gains of +4.8 to +6.2 over the baselines, underscoring its effectiveness in complex video understanding tasks.

Baseline uniform sampling underperforms by 3.0 to 9.5 points across all models, highlighting the importance of frame selection.
Furthermore, intermediate variants exhibit limited gains due to partial adherence to our proposed principles; neglecting query relevance, list-wise diversity, or sequentiality fundamentally restricts selection quality. 
Moreover, the ablation results further validate the design choices of \mname{}. For example, although the cosine similarity variant outperforms uniform sampling by 2.1 to 4.8 points, it still lags behind \textsc{mDP}$^3$ by the same margin. This difference underscores the advantage of our multi-kernel similarity in the reproducing kernel Hilbert space~(RKHS), which better captures high-dimensional feature relationships.

\begin{table}[t]
	\small
    \centering
    \begin{tabular}{c|lccc}
    \toprule
	 	 & \multicolumn{2}{r}{Video-MME} & MLVU & LVB$_{val}$ \\
    \midrule
    \multirow{6}{*}{\rotatebox{90}{VILA-V1.5}}
	     &\ + uniform & 47.5 & 46.3 & 47.1 \\
      	 &\ + SigLIP  & 50.6 & 53.9 & 46.4 \\
      	 &\ + \mname{} w. MGK       & 48.9 & 48.6 & 53.5 \\
      	 &\ + DPP w. CMGK           & 51.8 & 56.8 & 47.1 \\
         &\ + \mname{} w. cos. sim. & 50.2 & 54.4 & 48.8 \\
		 & \cellcolor{gray!30}\textbf{\textsc{mDP}$^3$} & \cellcolor{gray!30}\textbf{53.3} & \cellcolor{gray!30}\textbf{58.6} & \cellcolor{gray!30}\textbf{50.8} \\
	\midrule
	\multirow{6}{*}{\rotatebox{90}{MiniCPM-V2.6}}
	     &\ + uniform & 52.6 & 55.4 & 51.2 \\
      	 &\ + SigLIP  & 56.3 & 60.3 & 51.4 \\
      	 &\ + \mname{} w. MGK       & 52.6 & 62.0 & 52.6 \\
         &\ + \mname{} w. cos. sim. & 53.1 & 63.0 & 54.1 \\
         &\ + DPP w. CMGK           & 55.2 & 64.7 & 52.1 \\
		 & \cellcolor{gray!30}\textbf{\textsc{mDP}$^3$} & \cellcolor{gray!30}\textbf{58.0} & \cellcolor{gray!30}\textbf{66.6} & \cellcolor{gray!30}\textbf{57.1} \\
    \midrule
	\multirow{6}{*}{\rotatebox{90}{LLaVA-OV}}
	     &\ + uniform & 53.6 & 59.3 & 54.2 \\
      	 &\ + SigLIP  & 57.0 & 62.7 & 51.6 \\
      	 &\ + \mname{} w. MGK       & 54.9 & 63.3 & 52.7 \\
         &\ + DPP w. CMGK           & 56.4 & 68.6 & 52.8 \\
         &\ + \mname{} w. cos. sim. & 55.9 & 65.2 & 54.6 \\ 
		 & \cellcolor{gray!30}\textbf{\textsc{mDP}$^3$} & \cellcolor{gray!30}\textbf{59.6} & \cellcolor{gray!30}\textbf{69.8} & \cellcolor{gray!30}\textbf{59.0} \\
	\bottomrule
    \end{tabular}
	\caption{
    	Ablation study of \textsc{mDP}$^3$ across three Video-LLMs~(VILA-8B, MiniCPM-V2.6, LLaVA-OneVision) and benchmarks~(Video-MME, MLVU, LVB$_{val}$).
    	}
	\label{tab:additional_ablation}
\end{table}

\subsection{More Experiments with Various Selection Size}

\begin{figure*}[t]
	\centering
	\includegraphics[width=\linewidth]{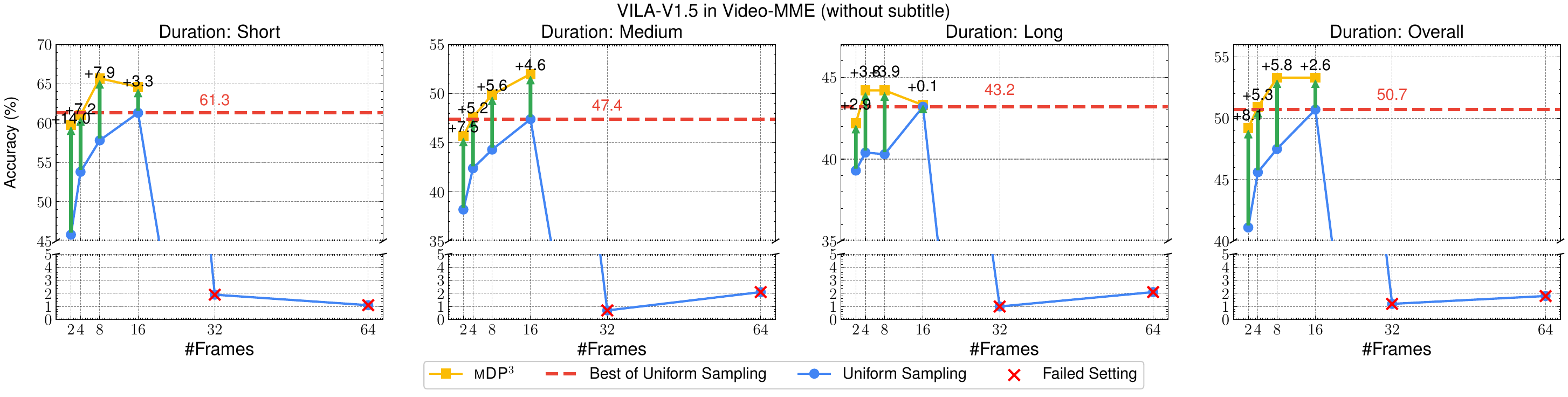}
	\caption{
	    Detailed accuracy (\%) on Video-MME (without subtitles) for VILA-V1.5 grouped by various durations, 
	    comparing uniform sampling with \mname{} selection from 128 candidate frames.
	}
   \label{fig:detailed_vila_var_frm}
\end{figure*}

\begin{figure*}[t]
	\centering
	\includegraphics[width=\linewidth]{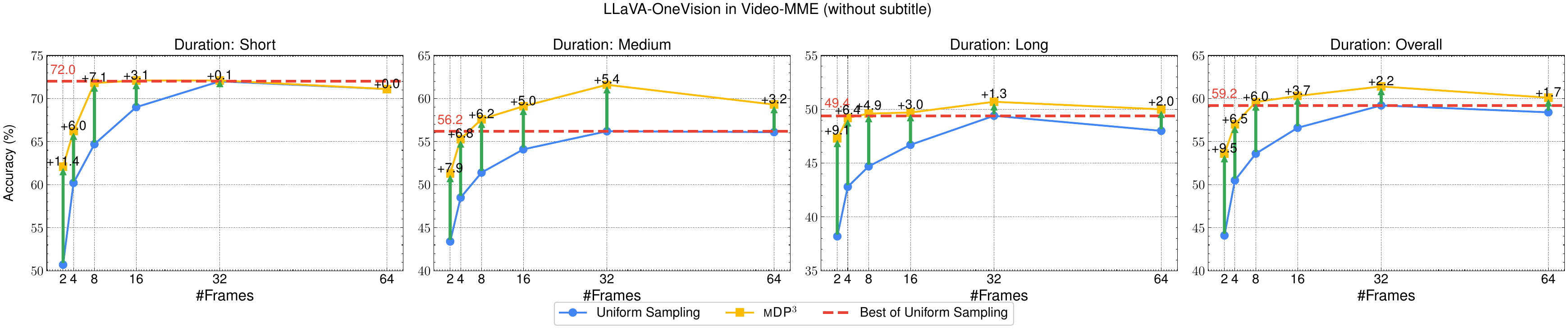}
	\caption{
	    Detailed accuracy (\%) on Video-MME (without subtitles) for LLaVA-OneVision grouped by various durations, 
	    comparing uniform sampling with \mname{} selection from 128 candidate frames.
	}
   \label{fig:detailed_llava_ov_var_frm}
\end{figure*}

We have reported the results for various selection sizes \(k\) in \cref{sssec:ravss} and \cref{fig:gain_on_videomme}. 
In this section, we provide additional results and more detailed discussions. 
In \cref{fig:detailed_llava_ov_var_frm,fig:detailed_vila_var_frm}, we present the performance of VILA-V1.5 and LLaVA-OneVision on Video-MME~(without subtitles) across various durations by varying the selection size \(k\).
The results show that, regardless of duration, VILA-V1.5’s performance improves with more input frames when \(k \leq 16\), while LLaVA-OneVision’s performance improves with additional frames when \(k \leq 32\). 
They employ the distinct context lengths and numbers of frames used during training. 
When the number of input frames exceeds the training stage, the input becomes out-of-distribution (OOD), leading to a performance drop. 
Specifically, when VILA-V1.5 receives more than 16 input frames, it experiences a catastrophic decline, whereas LLaVA-OneVision maintains some generalization ability when presented with more than 32 frames.
Although some Video-LLMs are trained with more frames to enhance long-video understanding, this approach is costly and the number of input frames cannot be increased indefinitely. 
Therefore, designing an effective frame selection algorithm is crucial.

There should be a meaningful discussion about the question: 
\textbf{``Are more input frames better for video understanding?''}
In \cref{tab:main_results}, we select 8 frames following the state-of-the-art frame selection approach, Frame-Voyage, settings to ensure a fair comparison. 
However, 8 frames may not be optimal for all Video-LLMs, benchmarks, and durations. 
Our answer to the question \textbf{``Are more input frames better for video understanding?''} is \textbf{No!}
We provide the following reasons:
\begin{enumerate}
	\item \textbf{Training Constraints:} Video-LLMs are typically trained on a limited number of frames. 
	When the number of input frames exceeds the training configuration, the model encounters OOD data, leading to degraded performance. 
	Although training with more frames can mitigate this issue, it is computationally expensive and unsustainable in the long run~(\ie, the number of input frames cannot be increased indefinitely.).
	
	\item \textbf{Inference Constraints:} Edge-deployed LLMs have limited resources, and proprietary models charge based on token usage. Consequently, processing excessive frames during inference stage is both resource-intensive and expensive.
	
	\item \textbf{Diminishing Returns:} 
	The marginal gain from adding frames diminishes exponentially (as indicated by the theoretical bound $\mathcal{O}((1/k)^k)$ in \cref{eq:detailed_bound} and supported by experiments in \cref{fig:gain_on_videomme,fig:detailed_llava_ov_var_frm,fig:detailed_vila_var_frm}), while the Transformer’s computational cost increases quadratically ($\mathcal{O}(n^2)$). Consequently, adding more frames is not cost-efficient.
	
	\item \textbf{Redundancy and Noise:} Incorporating additional frames may introduce redundant or irrelevant information, which can dilute salient features and add noise.

    \item \textbf{Latency and Real-Time Constraints:} In applications requiring real-time processing, increasing the number of frames can lead to higher latency, which may be unacceptable in time-sensitive scenarios~(as shown in \cref{fig:latency}).
\end{enumerate}
These considerations underscore the importance of developing an effective frame selection algorithm that balances performance improvements with computational efficiency and resource limitations. 
Instead of merely increasing the number of input frames, a well-designed selection strategy can extract the most informative frames, ensuring that the model focuses on quality over quantity and optimizes both accuracy and cost.

Additionally, the above discussion is supported across various durations, encompassing short, medium, and long video understanding scenarios.

\subsection{Qualitative Analysis}
We sample six representative cases from Video-MME~\cite{fu2024video}, illustrated in 
\cref{fig:case_study1,fig:case_study2,fig:case_study3,fig:case_study4,fig:case_study5,fig:case_study6}, 
to compare different frame selection methods, including
uniform sampling~(marked by~\includegraphics[height=0.8em]{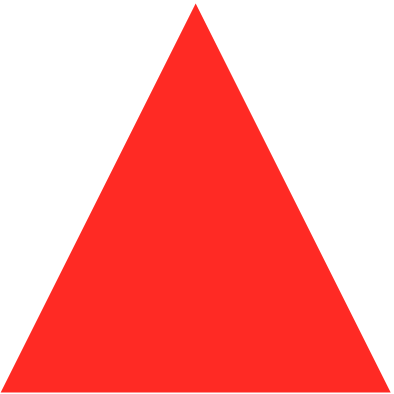} in the top right corner of frame)
and top-$k$ query-frame matching with SigLIP~\cite{zhai2023sigmoid}~(marked by~\includegraphics[height=0.8em]{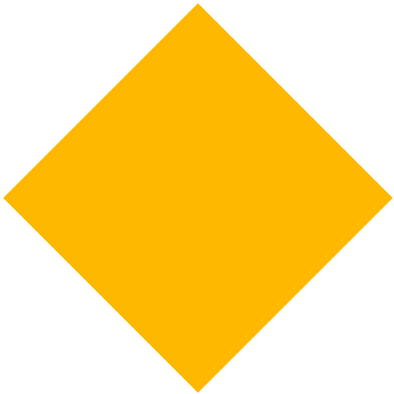}). 
The selection by \mname{} is marked by~\includegraphics[height=0.8em]{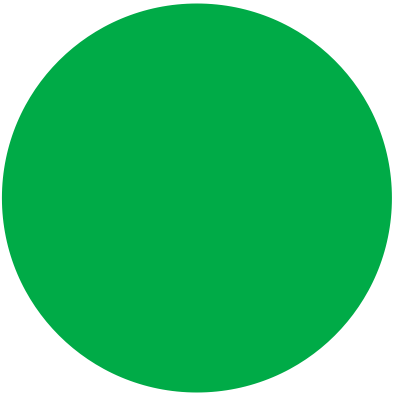}.
Besides, the ground-truth of option is colored by green.

We categorize the issues observed in the baseline frame selection process.
Especially, the point-wise top-$k$ query-frame matching with SigLIP, which is the best-performing baseline frame selection method in \cref{tab:abl}, but still falls significantly short compared to \mname{}. 
The observed issues are as follows:
\begin{enumerate}
    \item \textbf{Over-matching Specific Text} (\cref{fig:case_study1}):
    	As shown in \cref{fig:case_study1}, when asked, ``How many people are wearing ties in the video?'', the query-frame matching over-focuses on the keyword ``tie'', 
    	resulting in the selection of numerous duplicate frames featuring an individual with a prominent, visible tie.
    	This leads to the omission of frames where multiple people are wearing ties that are smaller or less noticeable.
    	In contrast, the selection by \mname{} demonstrates better balance between query relevance and frames diversity, effectively addressing this over-matching issue.
	\item \textbf{Failure of Counting Across Frames} (\cref{fig:case_study2,fig:case_study3}): 
		As shown in \cref{fig:case_study2,fig:case_study3}, when posed with questions such as, ``What is the total number of bird species visible in the video?'' or ``How many different kinds of animal faces appear in this video?'', these counting questions differ from the example in \cref{fig:case_study1}.
		Unlike the case in \cref{fig:case_study1}, the counted items are distributed across different frames and do not appear in a single frame.
		To answer accurately, it is necessary to select frames that include various relevant counted items.
		However, the baseline method of uniform sampling fails to identify specific counted items, while the query-frame matching approach struggles to avoid duplication, often focusing repeatedly on frames with the same item.
		In contrast, the frame selection by \mname{} demonstrates greater diversity, allowing for the inclusion of different frames with various items, thereby aiding MLLMs in accurately counting across frames.
	\item \textbf{Failure of Summarization} (\cref{fig:case_study4}):
		In this case, when the question is a summarization-type query such as ``What is the genre of this video,'' the frame selection requires a comprehensive representation of the entire video rather than focusing on specific event clips.
		Since there is no specific key text to match frames, the query-matching fails, performing even worse than uniform sampling.
		In contrast, the frame selection by \mname{} demonstrates a global understanding of the entire video, exhibiting good diversity and assisting MLLMs in summarizing the content effectively.	
	\item \textbf{Failure of Reverse Question Answering} (\cref{fig:case_study5}):
		This case is particularly interesting as it focuses on identifying events or items that do \emph{NOT} appear in the video, such as ``Which of the following elements does not appear in the video?'' 
		This type of reverse QA poses a significant challenge for query-frame matching since there is no key text for matching.
		However, \mname{} demonstrates strong performance, ensuring diversity in the selected frames and providing a more comprehensive representation of the video.
	\item \textbf{Failure of Transition Awareness} (\cref{fig:case_study6}):
		As shown in \cref{fig:case_study6}, when asked, ``How many times does the interviewed girl appear in the video?'', this question represents a special type of counting task: it not only requires counting across frames but also identifying the distinct \emph{number of times} the interviewed girl appears.
		While the concept of the ``interviewed girl'' is singular, the actual item to be counted is the ``number of appearances'', making temporality and sequentiality crucial.
		As shown in \cref{fig:case_study6}, the query-frame matching fails to recognize the transitions between appearances of the interviewed girl.
		This oversight leads to missing frames between appearances, resulting in multiple occurrences being merged into a single one.
		Consequently, MLLMs cannot accurately count the number of appearances with such a selection. 
		In contrast, \mname{} effectively captures the sequential nature of the video and recognizes transitions between appearances, enabling accurate counting.
\end{enumerate}

Additionally, this case study not only provides a qualitative analysis of various frame selection methods,
but also reveals the limitations of baseline frame selection and highlights the strengths of \mname{}. 
Besides, it raises several challenges in VidQA, and serves as a guide for constructing a more comprehensive benchmark to evaluate the video understanding capabilities of MLLMs.
This study is highly valuable to the technology community focused on MLLMs.

\begin{figure*}[t]
	\centering
  	\includegraphics[width=0.85\linewidth]{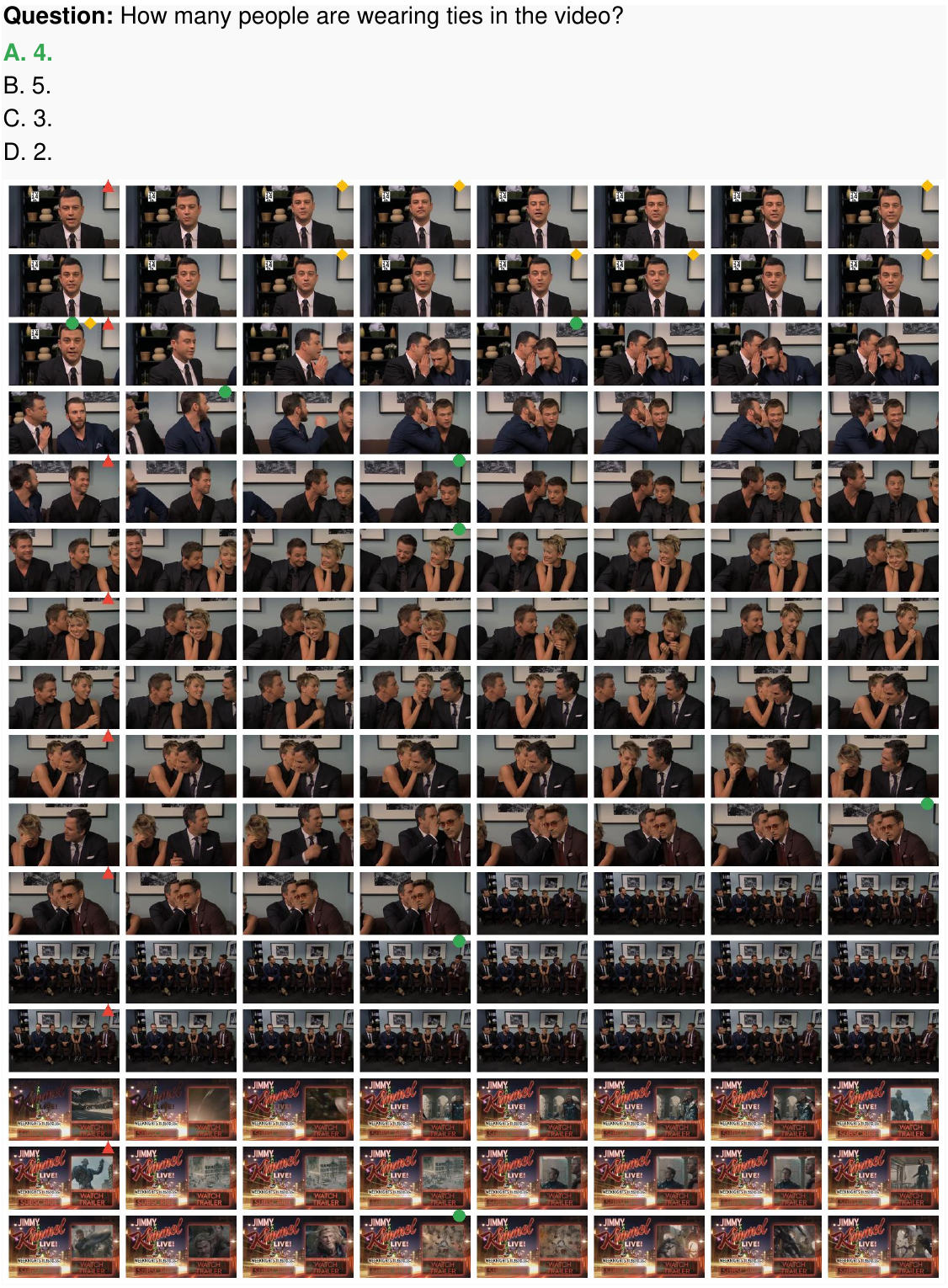}
	\caption{
		\includegraphics[height=0.8em]{./appendix/figures/t.png}: uniform sampling;
		\includegraphics[height=0.8em]{./appendix/figures/d.png}: top-$k$ query-frame matching with SigLIP;
		\includegraphics[height=0.8em]{./appendix/figures/r.png}: \mname{}.
		Over-matching of the keyword ``tie'' leads to duplicate frames being selected, omitting frames where multiple people wear ties. \mname{} addresses this issue by balancing query relevance with frame diversity.	}
	\label{fig:case_study1}
\end{figure*}

\begin{figure*}[t]
	\centering
  	\includegraphics[width=0.85\linewidth]{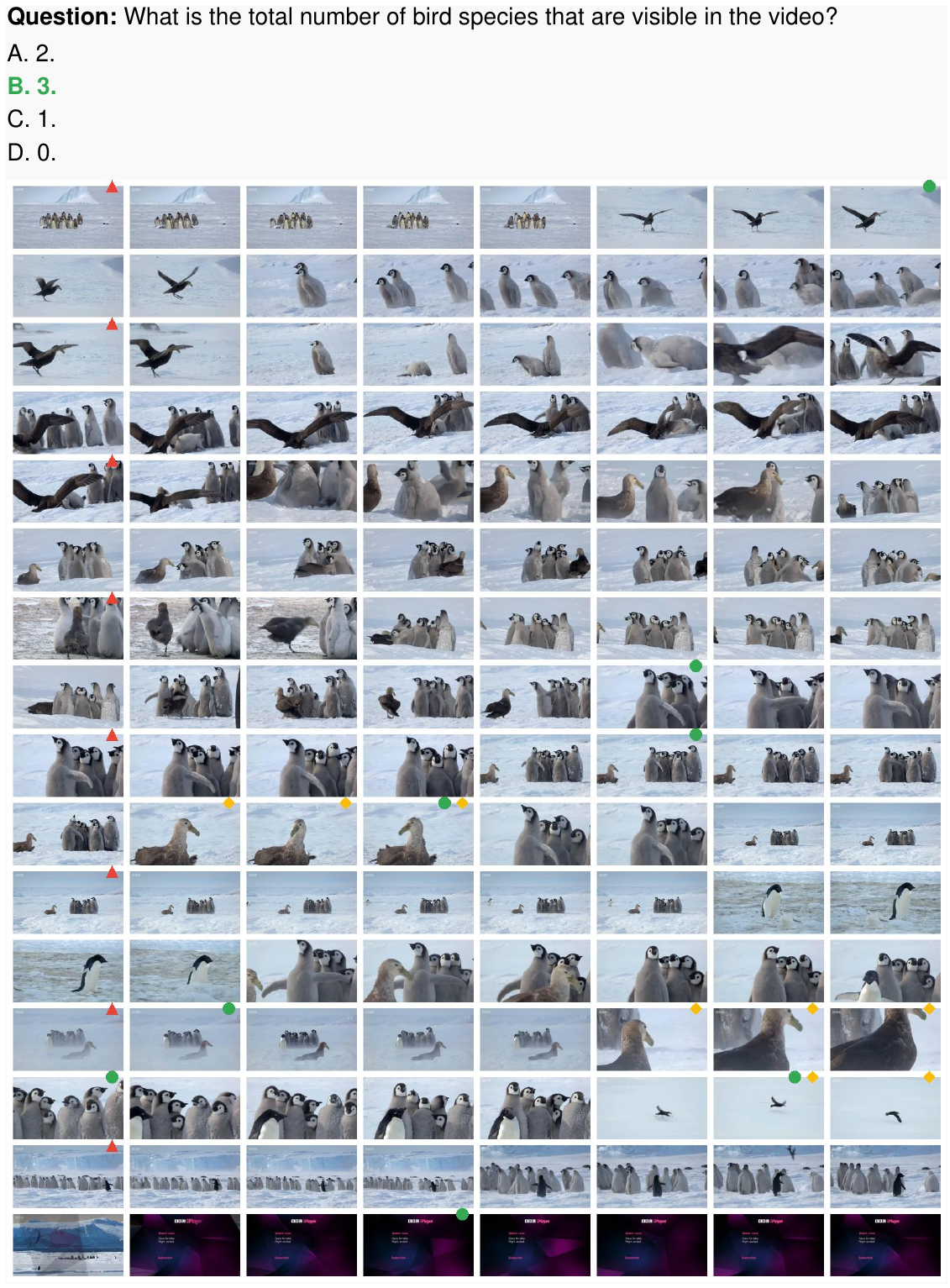}
	\caption{
		\includegraphics[height=0.8em]{./appendix/figures/t.png}: uniform sampling;
		\includegraphics[height=0.8em]{./appendix/figures/d.png}: top-$k$ query-frame matching with SigLIP;
		\includegraphics[height=0.8em]{./appendix/figures/r.png}: \mname{}.
		In counting tasks across frames, such as ``How many bird species or animal faces are in the video?'', uniform sampling and query-frame matching struggle with item duplication. 
		\mname{} improves diversity, aiding in accurate counting across frames.
	}
	\label{fig:case_study2}
\end{figure*}\begin{figure*}[t]
	\centering
  	\includegraphics[width=0.85\linewidth]{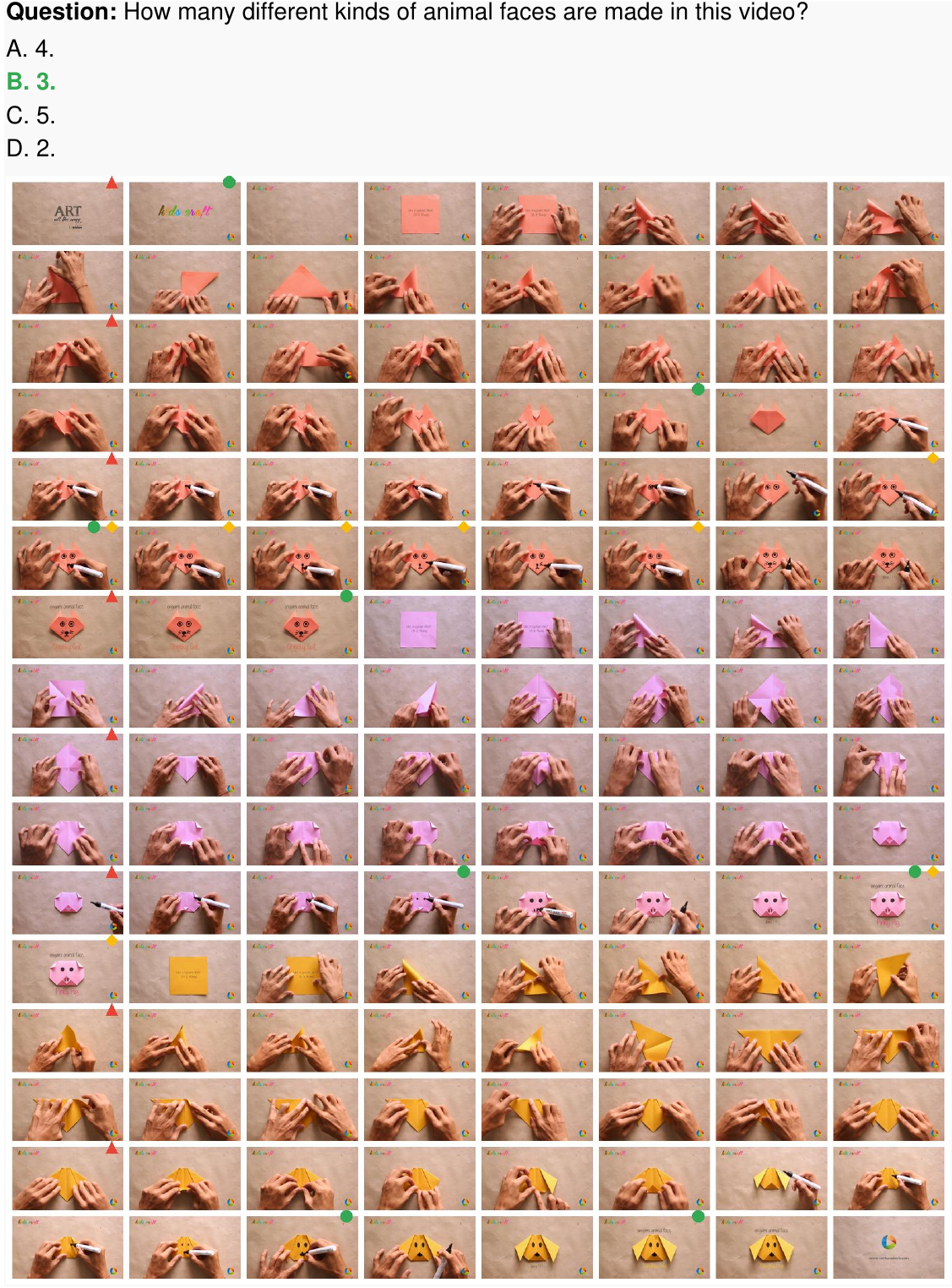}
	\caption{
		\includegraphics[height=0.8em]{./appendix/figures/t.png}: uniform sampling;
		\includegraphics[height=0.8em]{./appendix/figures/d.png}: top-$k$ query-frame matching with SigLIP;
		\includegraphics[height=0.8em]{./appendix/figures/r.png}: \mname{}.
		In counting tasks across frames, such as ``How many different kinds of animal faces are made in this video?'', uniform sampling and query-frame matching struggle with item duplication. 
		\mname{} improves diversity, aiding in accurate counting across frames.
		}
	\label{fig:case_study3}
\end{figure*}

\begin{figure*}[t]
	\centering
  	\includegraphics[width=0.85\linewidth]{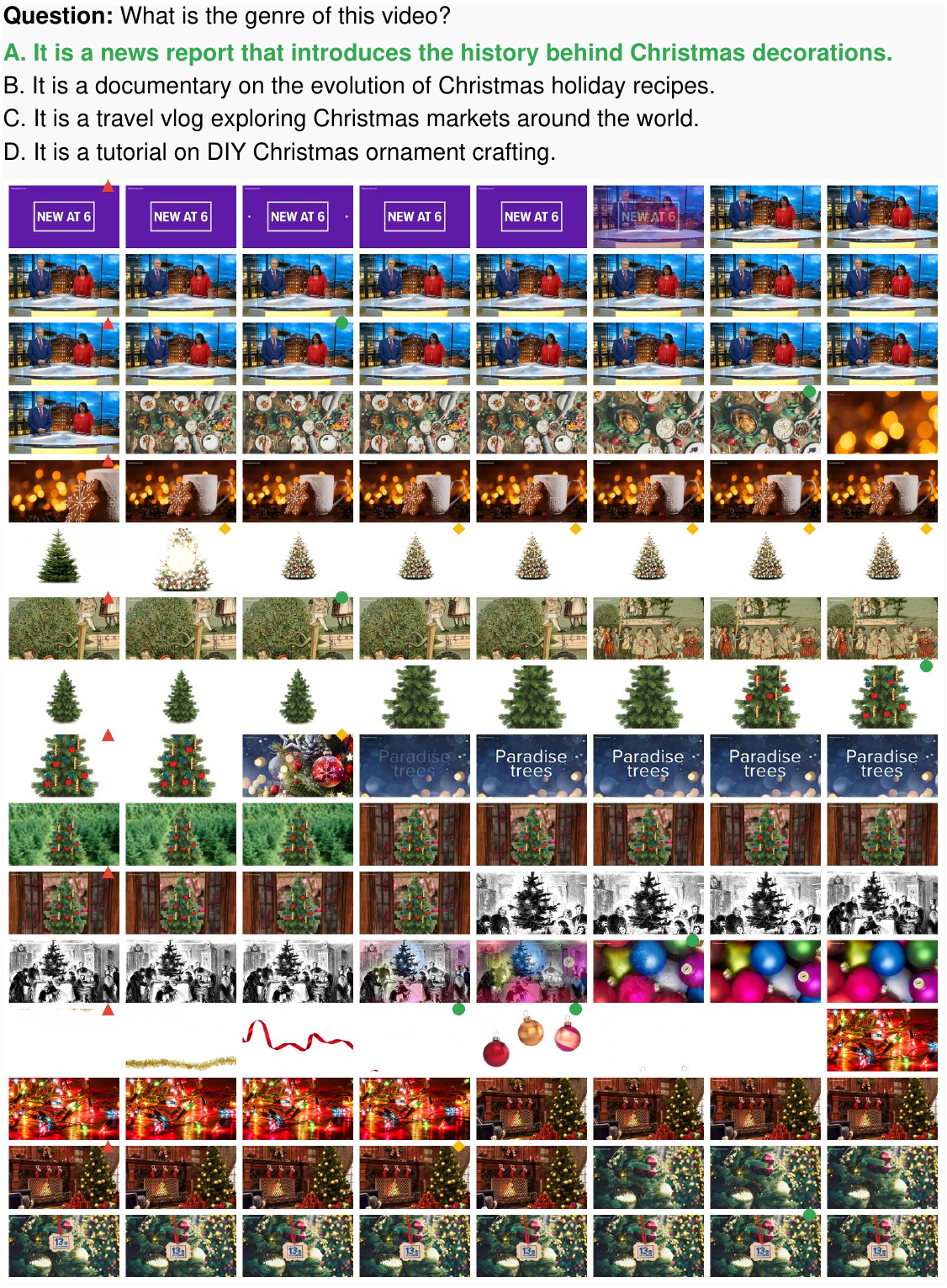}
	\caption{
		\includegraphics[height=0.8em]{./appendix/figures/t.png}: uniform sampling;
		\includegraphics[height=0.8em]{./appendix/figures/d.png}: top-$k$ query-frame matching with SigLIP;
		\includegraphics[height=0.8em]{./appendix/figures/r.png}: \mname{}.
		For summarization queries like ``What is the genre of this video?'', query-frame matching fails to represent the entire video. \mname{} shows a global understanding of the video, enhancing diversity and assisting in summarization.
		}
	\label{fig:case_study4}
\end{figure*}

\begin{figure*}[t]
	\centering
  	\includegraphics[width=0.85\linewidth]{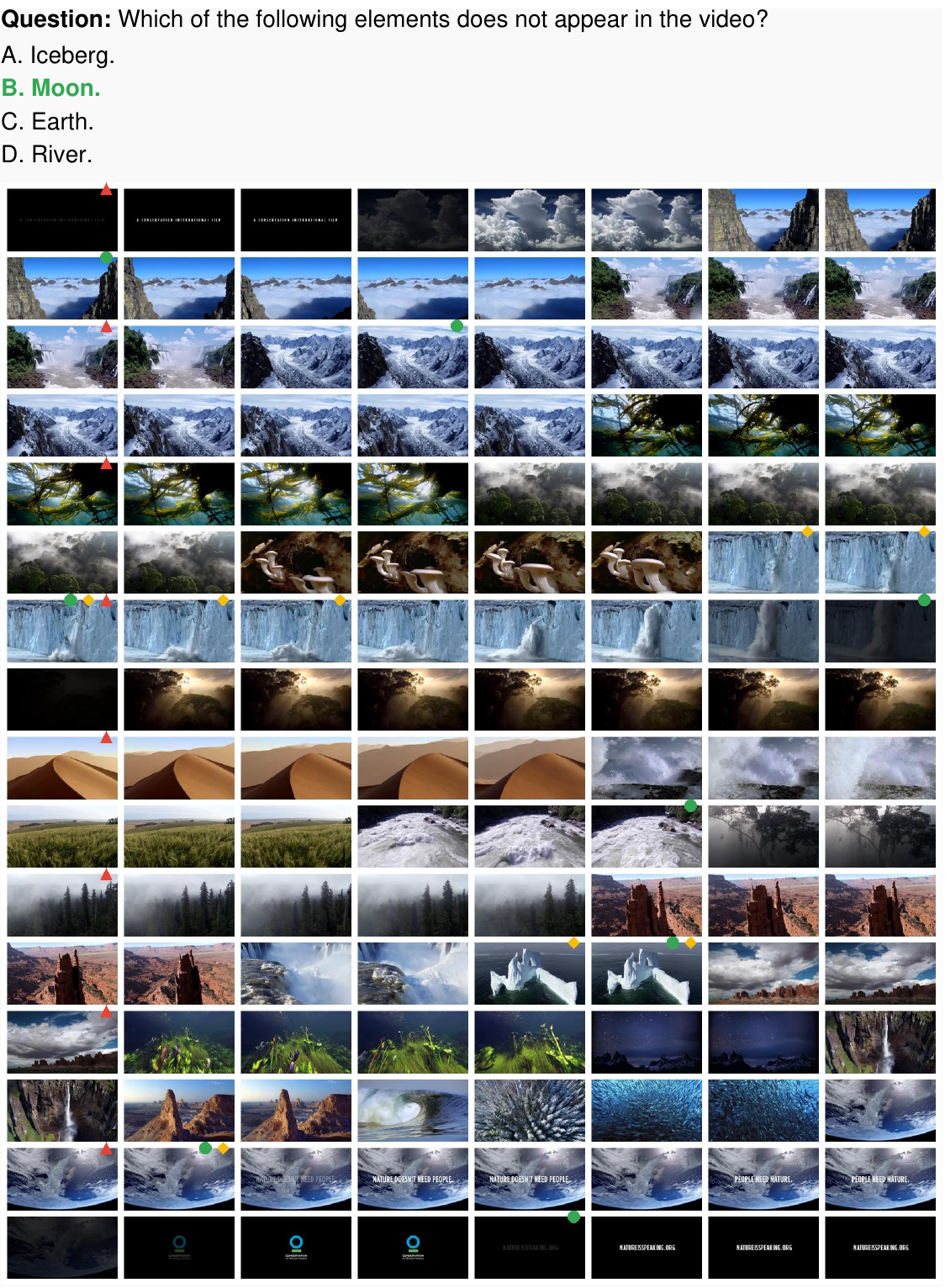}
	\caption{
		\includegraphics[height=0.8em]{./appendix/figures/t.png}: uniform sampling;
		\includegraphics[height=0.8em]{./appendix/figures/d.png}: top-$k$ query-frame matching with SigLIP;
		\includegraphics[height=0.8em]{./appendix/figures/r.png}: \mname{}.
		Reverse QA, such as ``Which elements \emph{DO NOT} appear in the video?'', presents a challenge due to the lack of a specific key text for matching.
		\mname{} excels by ensuring diversity in selected frames and providing a comprehensive video representation.
	}
	\label{fig:case_study5}
\end{figure*}

\begin{figure*}[t]
	\centering
  	\includegraphics[width=0.85\linewidth]{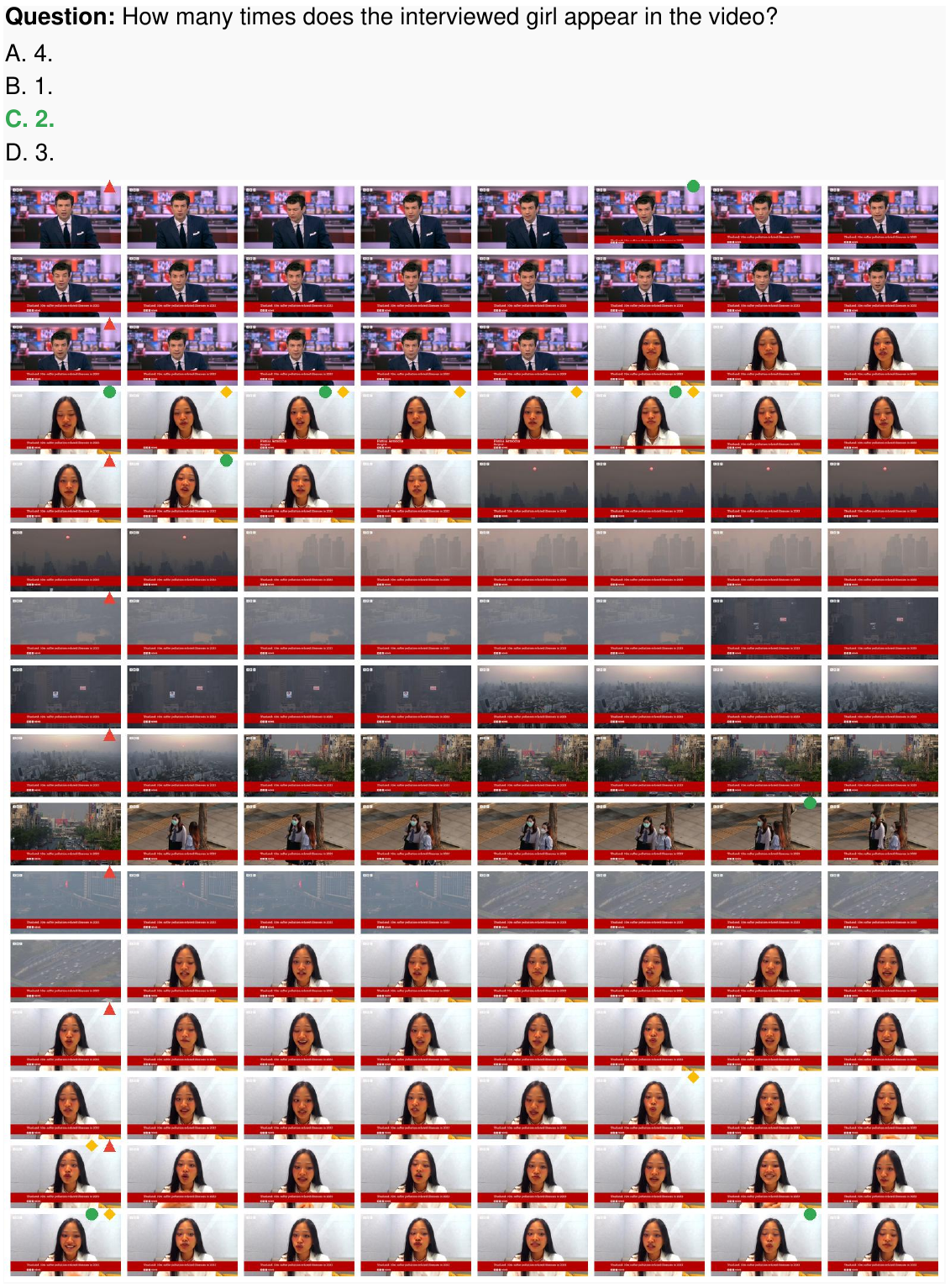}
	\caption{
		\includegraphics[height=0.8em]{./appendix/figures/t.png}: uniform sampling;
		\includegraphics[height=0.8em]{./appendix/figures/d.png}: top-$k$ query-frame matching with SigLIP;
		\includegraphics[height=0.8em]{./appendix/figures/r.png}: \mname{}.
		For questions like ``How many times does the interviewed girl appear?'', query-frame matching fails to capture the transitions between appearances. \mname{} accurately counts the number of appearances by considering sequentiality.
	}
	\label{fig:case_study6}
\end{figure*}